\documentclass[twocolumn]{article}
\usepackage[a4paper, portrait, left=.7in, right=.7in, bottom=1.1in, top=1.1in]{geometry}
\usepackage{enumitem}
\usepackage{amsmath}
\usepackage{amsfonts}
\usepackage{amssymb}
\usepackage{mathtools}
\usepackage{amsthm}
\usepackage{dsfont}
\usepackage{xurl}
\usepackage[hidelinks]{hyperref}

\usepackage{caption}
\usepackage{subcaption}

\usepackage[ruled,vlined,linesnumbered]{algorithm2e}
\SetAlgoSkip{bigskip}

\newtheorem{thm}{Theorem}[section]

\newtheorem{lemma}[thm]{Lemma}
\newtheorem{cor}[thm]{Corollary}

\theoremstyle{definition}
\newtheorem{rmrk}[thm]{Remark}

\newcommand{\negspace}{\vspace{-0.0cm}}

\newcommand{\cN}{\mathcal{N}}
\newcommand{\cA}{\mathcal{A}}
\newcommand{\cX}{\mathcal{X}}

\newcommand{\cM}{\mathcal{M}}

\newcommand{\cL}{\mathcal{L}}
\newcommand{\cB}{\mathcal{B}}

\newcommand{\cS}{\mathcal{S}}

\newcommand{\de}{\mathrm{d}}
\newcommand{\LCLIP}{\cL_\text{CLIP}^{(\theta)}}
\newcommand{\E}{\mathbb{E}}
\newcommand{\cU}{\mathcal{U}}
\newcommand{\R}{\mathbb{R}}
\newcommand{\N}{\mathbb{N}}
\newcommand{\unclipped}{\mathcal{U}_\theta^{(\varepsilon)}}
\newcommand{\unclippedbrd}{\mathcal{V}_\theta^{(\varepsilon)}}
\newcommand{\rk}[1]{\text{rank}\left(#1\right)}

\newcommand{\indicator}[1]{\mathds{1}_{#1}}
\newcommand{\code}[1]{\texttt{\hyphenchar\font=`\-\relax #1}}
\newcommand{\jointmeasure}{\nu^T_{\pi_{\theta_{old}}}}

\newcommand{\Esa}{\E_{s,a}}
\newcommand{\Esxi}{\E_{s,\xi}}
\newcommand{\Es}{\E_s}
\newcommand{\EB}{\E_B}

\renewcommand{\P}{\mathbb{P}}
\DeclareMathOperator*{\argmin}{arg\,min\;}

\usepackage[most]{tcolorbox}

\usepackage[numbers]{natbib}
\usepackage{authblk}

\allowdisplaybreaks
\title{\bfseries{Bridging Exploration and Loss of Plasticity\\in On-Policy Reinforcement Learning}\footnote{Preprint. Under review.}}

\author[1,2]{Andrea Vasco Grieco}
\author[2]{Evelyn D'Elia}
\author[2]{Edoardo Caldarelli}
\author[1,3]{Raffaello Camoriano}

\affil[1]{DAUIN, Politecnico di Torino, 10129 Turin, Italy}
\affil[2]{SAIL Research Unit, Istituto Italiano di Tecnologia, 16163 Genoa, Italy}
\affil[3]{Rehab Technologies Lab, Istituto Italiano di Tecnologia, 16163 Genoa, Italy}
\affil[ ]{\textbf{Correspondence to:} \texttt{andreavasco.grieco@studenti.polito.it}}
\date{}

\begin{document}
\maketitle

\begin{abstract}\vspace{-0.0cm}
Deep reinforcement learning (RL) algorithms have been observed to lose the ability to adaptively learn from new data, a phenomenon termed Loss of Plasticity (LoP).
While many metrics have been shown to correlate with it, the causal mechanisms and the geometric conditions from which it can arise have not been clearly identified.
We analyze LoP in Proximal Policy Optimization with isotropic Gaussian policies, the state-of-the-art formulation for many continuous action spaces in RL. 
We report multiple metrics which characterize the effect that the policy standard deviation $\sigma$ has on LoP, and we also suggest related explanations for the effectiveness of widely adopted LoP mitigation techniques.
We observe a significant effect of $\sigma$ on the extent of LoP, and accompany this with both theoretical and experimental supporting arguments. 
Our experimental validation consists of LoP benchmark evaluations against standard baselines as well as real-world policy deployment on a 7-DoF robotic arm.
\end{abstract}
\section{INTRODUCTION}\negspace
In the last decade, reinforcement learning (RL) has achieved impressive results in a broad set of domains, ranging from natural language processing~\citep{ouyang2022rlhfllms,Guo2025} to robotics ~\citep{pmlr-v162-nikishin22a,haarnoja2024soccer}.
A popular class of RL algorithms are policy gradient (PG) methods, originating with REINFORCE~\citep{Williams1992} and achieving improved success rates by integrating natural gradients~\citep{amari1998naturalgradients,amari1998whyng,kakade2001naturalpolicygradient}. 
A relevant RL subclass are on-policy gradient methods, in which the samples used for training are obtained by collecting rollouts with the current policy.
The development of Proximal Policy Optimization (PPO)~\citep{schulman2017proximalpolicyoptimizationalgorithms} represents a landmark  in this field, due to its relative simplicity, stability, and sample efficiency.
PPO is used to solve complex robotics tasks~\citep{hwangbo2019leggedrobots,rudin2021learning,lee2020uneventerrain} and is considered the state-of-the-art technique for on-policy RL.
To model problems in which the agent acts in a continuous action space, a typical approach is to adopt Gaussian policies~\citep{Benbrahim1997,peters2007machine}, even though alternative exploration strategies, such as Policy Gradients with
Parameter-based Exploration (PGPE)~\citep{sehnke2010} or natural evolution strategies~\citep{wierstra2014nes,salimans2017evolutionstrategiesscalablealternative} have also been proposed.\looseness=-1

\negspace
Despite its success, PPO is not well studied theoretically, and its use is accompanied by a proliferation of heuristics~\citep{Engstrom2020Implementation,andrychowicz2021what,huang2022ppo_details} and known technical issues, such as the premature collapse of the Gaussian policy's standard deviation $\sigma$~\citep{hamalainen2020ppocma}. 
In fact, most of the theoretical research focuses on vanilla PG methods such as REINFORCE, either deriving the optimal learning rate and batch size for safe PG \citep{pirotta2013,papini2017batchsize,papini2020learningspeedstability}, or showing that PG's gradients with Gaussian policies suffer from increased variance as $\sigma\to0$~\citep{zhao2011policygradient}.
However, results for PPO are very limited, with~\cite{chingrevisitingPPO} only informally mentioning that gradients scale inversely with respect to $\sigma$, yielding large updates and instability.

\negspace
Another issue observed in deep RL is the diminished ability of neural networks to learn new objectives after being trained, commonly referred to as Loss of Plasticity (LoP).
Indeed, \cite{dohare2022continual,Dohare2024} observe that PPO asymptotically collapses and that this process is accelerated under non-stationarity.
Some explanations include the lack of feature representation richness 
\citep{igl2021transient,kumaragarwal2021implicit,lyle2022infer}, the sharpening of the loss landscape ~\cite{pmlr-v202-lyle23b, lewandowski2024curvature}, or dormant neurons~\cite{pmlr-v202-sokar23a}. Additionally,~\cite{pmlr-v162-nikishin22a,nauman2024bitterlesson} study gradient norms  correlated with this phenomenon. Finally,~\cite{juliani2024onpolicylop} studies LoP in PPO specifically, but only focuses on discrete actions.

To mitigate LoP, multiple strategies have been proposed. ~\cite{pmlr-v202-lyle23b,asadi2023resetting,ellis2024adam,Dohare2024} show that tweaking the Adam optimizer~\citep{kingma2015adam}, either by setting the first and second moment coefficients $\beta_1$ and $\beta_2$ to the same value, or by partially resetting the optimizer parameters, is effective in reducing LoP. 
Furthermore,~\cite{wang2026tracking} shows that increasing the strength of entropy regularization under non-stationarity is effective in improving task adaptation.~Finally,~\cite{joudaki2026barriers} proposes formalized mathematical conditions that could explain LoP, showing that there exist \textit{LoP manifolds} in which the weights of a neural network would be trapped after entering, with any gradient descent method regardless of the objective.
\paragraph{Contributions} While LoP has been thoroughly studied in recent years, there is no consensus on its causes. Furthermore, for many of the proposed causes, there exist counter-examples. 
For this reason, we focus on a specific setting, PPO with Gaussian policies, and show formally and experimentally a novel link between the policy's standard deviation $\sigma$ and LoP. In summary, our contributions are:
\begin{itemize}[noitemsep,leftmargin=*,topsep=0pt]
    \item The derivation of bounds linking multiple LoP metrics to $\sigma\to0$, formally proving and numerically verifying that under mild integrability assumptions the gradient norm and the Hessian sharpness diverge, while the Hessian approximate rank collapses, as $\sigma\to0$;
    \item Experiments on an established benchmark showing that $\sigma$ is related to the observed LoP phenomenon and that the effectiveness of existing techniques counteracting LoP is explained by their indirect influence on $\sigma$;
    \item Task-adaptation experiments with real-world robotic deployment demonstrating  that our results also apply to a complex scenario with a continuous action space.
\end{itemize}
\negspace
\section{BACKGROUND}
\negspace
\paragraph{Markov Decision Processes}
Let $\Delta(\cdot)$ be the space of probability measures over a set and let $\cM=(\cS,\cA,T,p,R,\rho_0)$ be an episodic Markov Decision Process (MDP), where $\cS$ is the state space, $\cA\in\R^n$ is the action space with dimensionality $n$, $T$ is the horizon, $p(\cdot\,|s,a\,)\colon \cS\times\cA\to\Delta(\cS)$ is the transition probability, $R\colon\cS\times\cA\to\R$ is the reward function, and $\rho_0\in\Delta(\cS)$ is the initial state probability distribution. 
A \emph{policy} is a map $\pi(\cdot\,|\,s)\colon\cS\to\Delta(\cA)$ that selects a probability distribution over $\cA$ given a state $s$.
Let $G_t\coloneqq \sum_{k=0}^{T-t-1} \gamma^kR(s_{t+k}, a_{t+k}, s_{t+k+1})$, for $k\in\N$, denote the discounted cumulative reward achieved after step $t$ according to a discount factor $\gamma\in[0,1]$ by sampling trajectories according to $\pi$. 
Then, it is possible to define the value function $V^{\pi}\colon\cS\to\R$ as $V^{\pi}_t(s) \coloneqq \E_{\pi}[G_t\mid S_t=s]$ and state-action value function $Q^\pi\colon\cS\times\cA\to\R$ as  $Q^{\pi}_t(s,a) \coloneqq \E_{\pi}[G_t \mid s_t=s, a_t=a]$.

In the following, we assume $\pi_\theta(\cdot |s)$ to be a deep isotropic Gaussian policy parameterized by $\theta\coloneqq(\phi,\sigma)$, so that actions given $s\in \cS$ are sampled from $\cN(\mu_\phi(s),\sigma^2I)$, with $\mu_\phi$  a neural network parameterized by $\phi$ and $\sigma\in\R_+$ the policy's standard deviation.\negspace
\paragraph{Proximal Policy Optimization} Our work is focused on Proximal Policy Optimization (PPO)~\citep{schulman2017proximalpolicyoptimizationalgorithms}.
Given a parameterization $\theta_{old}=(\phi_{old},\sigma_{old})$ and the corresponding policy $\pi_{\theta_{old}}$, PPO aims to maximize w.r.t.\ $\theta$ the objective  $\Esa[\LCLIP(s,a)]$, with $\LCLIP(s,a)$ defined as
\begin{equation}\label{eq:lclip}
\min\left\{r_\theta(s,a)\hat A(s,a), \mathrm{clip}(r_\theta, 1-\varepsilon, 1+\varepsilon)\hat A(s,a)\right\},
\end{equation}

where $r_\theta(s,a)\coloneqq \pi_\theta(a\,|\,s)/\pi_{\theta_{old}}(a\,|\,s)$ is the policy density ratio and $\hat A(s,a)$ is an estimate of the advantage $A(s,a)=Q^{\pi_{\theta_{old}}}(s,a)-V^{\pi_{\theta_{old}}}(s)$ of a given action $a$ taken in state $s$ following the previous policy $\pi_{\theta_{old}}$. 
We adopt the generalized advantage estimation (GAE)~\citep{schulman2016gae} to estimate advantages. 
\negspace
\paragraph{Probability measures}
Let $\pi_\theta(\cdot|s)\colon\cS\to\cA\subseteq\R^n$  be a stochastic policy, $T\in\N$ be a finite horizon and $\cB(\cdot)$ denote the Borel sigma-algebra. 
Analogously to~\cite{schulman2015trpo}, we observe that $\pi_\theta$ naturally induces two probability distributions $\rho^T_{\pi_\theta}\colon\cB(\cS)\to[0,1]$ and $\nu^T_{\pi_\theta}\colon\cB(\cS\times \cA)\to[0,1]$ on the state space $S$ and on the state-action space $\cS\times \cA$. 
Formally, for every  $S\in\cB(\cS)$, $U\in\cB(\cS\times \cA)$ we set
$
    \rho^T_{\pi_\theta}(S)\coloneqq \frac 1T\sum_{t=0}^T \P(s_t\in S \mid \pi_\theta),\
    \nu^T_{\pi_\theta}(U)\coloneqq \frac 1T\sum_{t=0}^T \P((s_t,a_t)\in U \mid \pi_\theta),
$
where $\P(\omega\mid\pi_\theta)$ denotes the probability of the event $\omega$ by sampling actions according to the policy $\pi_\theta$.
Throughout this paper, we will always assume that the event 
$
        \unclippedbrd \coloneqq \{(s,a) \,:\, r_\theta(s,a)=1+\varepsilon, \hat A(s,a) \geq 0\}\cup\\
        \{(s,a) \,:\, r_\theta(s,a)=1-\varepsilon, \hat A(s,a) \leq 0\}
$
has null probability under the probability measure $\nu^T$, so that gradients are defined almost everywhere.
\negspace
\paragraph{Loss of Plasticity} The phenomenon of Loss of Plasticity (LoP) is defined as the observed diminished ability of neural networks to fit new objectives after being trained on a different goal. Many explanations have been proposed for LoP, but for most of them the literature also reports counter-examples. Thus, a full understanding of the phenomenon has not been achieved yet. 
In particular, factors including gradient magnitude~\citep{pmlr-v162-nikishin22a,nauman2024bitterlesson} and Hessian properties such as the sharpness (the magnitude of the largest Hessian eigenvalue) or the Hessian stable rank~\citep{pmlr-v202-lyle23b,lewandowski2024curvature} have been proposed as possible causes for LoP. 
Our goal is to show that, in PPO with Gaussian policies, the behavior of these factors can be linked to the collapse of $\sigma$. We propose this parameter as one of the main contributors to the observed LoP.
\negspace
\section{THEORETICAL ANALYSIS}\label{sec:optimization_landscape}\negspace
This section introduces a connection between the collapse of $\sigma$ and several metrics associated with LoP. It derives lower bounds for gradient norm and Hessian sharpness divergence, and shows that the Hessian approximate rank collapses as $\sigma\to0$. 
This serves the purpose of deriving an effective learning rate scaling and explains the partial success of tuning Adam with $\beta_1=\beta_2$ in mitigating LoP.\negspace
\subsection{Divergence of the gradient norm}\negspace
The following theorem shows that, under some integrability conditions, the norm of a gradient w.r.t.\ $\phi$ estimated over a batch diverges in expectation as $\sigma\to0$. 
\begin{thm}\label{thm:minibatch_gradient_norm}
Let $B\coloneqq\{(s_i,a_i)\}_{i=1}^N$ denote a batch of $N$ i.i.d.\ state-action pairs, each sampled according to the probability measure $\jointmeasure$ induced by the policy $\pi_{\theta_{old}}$, so that $B\sim(\jointmeasure)^N$.
Let
\begin{equation}\label{eq:bar_g_theta}
    \bar g_\theta (B)
    \coloneqq
    \frac1N\sum_{i=1}^N \nabla_\phi \LCLIP(s_i,a_i)
\end{equation}
be the gradient computed on $B$. If $\exists K_\theta>0$ such that
\begin{equation}\label{eq:K_theta}
    0 <\E_B[\|\bar g_\theta(B)\|^4]
    \leq
    K_\theta\E_B[\|\bar g_\theta(B)\|^2]^2 < \infty ,
\end{equation}
and the following ratio $\kappa_\theta$ exists and is finite,
    \begin{equation}\label{eq:k_theta}
        \kappa_\theta\coloneqq \frac{\Esa[\|\nabla_\phi\LCLIP(s,a)\|^2] }{\E_{s,a\sim \nu^T_{\pi_{\theta_{old}}}}[\|\nabla_\phi\log\pi_\theta(a|s)\|^2]},
    \end{equation}
then, letting $J_{\mu_\phi}(s)$ be the Jacobian w.r.t. $\phi$ of $\mu_\phi(s)$, and $\|\cdot\|_F$ the Froboenius norm,
\begin{equation}\label{eq:minibatch_bound}
    \E_B[\|\bar g_\theta(B)\|]
    \geq
    \frac{\sqrt{\kappa_\theta}}{4K_\theta\sqrt{2N}}
    \,
    \frac{\sigma_{old}}{\sigma^2}
        \E_s[\|J_{\mu_\phi}(s)\|_F^2]
    ^{1/2}.
\end{equation}
\end{thm}
\paragraph{Proof sketch.} We start by explicitly computing the gradient of $\LCLIP$ and showing that $\nabla_\phi\log\pi_\theta(a|s) = \frac 1\sigma J_{\mu_\phi}^\top(s)\xi$ with $\xi=\xi(s,a)\coloneqq(a-\mu_{\phi}(s))/{\sigma}$ and observe that $\xi$ is normally distributed (but not standard in general). 
By adapting the known equation $\E_\zeta[\|A\zeta\|^2] = \E_\zeta[\|A\|^2_F]$ to non-standard $\xi$, valid for all $\zeta\sim\cN(0,I)$, we get a bound similar to~\eqref{eq:minibatch_bound} with the RHS squared, that holds for $\E_{s,a}[\|\nabla_\phi\LCLIP(s,a)\|^2]$. 
The assumption of i.i.d.\ samples allows us to bound the square minibatch gradient norm with $\frac{1}{N} \Esa[\|\nabla_\phi\LCLIP(s,a)\|^2].$
Using the kurtosis-like assumption~\eqref{eq:K_theta}, it is possible to take the square root inside the expectation $\E[\cdot]$ and on the corresponding RHS bound. A complete proof is provided in Appendix~\ref{app:proof_minibatch_thm}. 

Theorem~\ref{thm:minibatch_gradient_norm} draws a clear link between the gradient norm over a batch and $\sigma$, stating that if $\sigma_{old}\approx\sigma$, then $\E_B[\|\bar g_\theta(B)\|]\gtrsim \sigma^{-1} \E_s[\|J_{\mu_\phi}(s)\|_F^2]$. Observe that, during optimization, there is no guarantee that 
$J_{\mu_\phi}(s)\to 0$, as we know that, at an optimum, it only holds that $\nabla_\phi\LCLIP\to0$. Therefore, $J_{\mu_\theta}^\top(s)\nabla_{\mu} \LCLIP \to 0$, simply requiring that the columns of $J_{\mu_\theta}(s)$ should be orthogonal to $\nabla_{\mu} \LCLIP$, but not necessarily approach zero. 
This implies that one of the behaviors typically associated with LoP may actually be caused by the reduction of the policy's exploration capabilities, which are governed by $\sigma$. 
\negspace
\subsection{Divergence of the sharpness}\negspace
This section shows that, as $\sigma\to0$, the largest eigenvalue of the Hessian diverges proportionally to $\sigma^{-2}$ in expectation. 
This has strong implications over the optimization landscape, given that a stable optimization requires annealing the learning rate accordingly, in absence of the \textit{edge of stability} phenomenon~\citep{cohen2021gradient}. 
Moreover, higher Hessian sharpness degrades network performance, correlating with poorer generalization~\citep{foret2021sharpnessaware} and exacerbated LoP~\citep{pmlr-v202-lyle23b,lewandowski2024curvature}.

Let  $\xi =\xi(s,a)\coloneqq(a-\mu_{\phi}(s))/{\sigma}$, then a direct computation, included in Appendix~\ref{app:hessian_decomposition}, yields the following Hessian decomposition
\begin{equation}\label{eq:hessian_implicit}
    H_\theta(s,a)= W_\theta(s,a)\left(\frac{1}{\sigma^2}G_\phi(s,\xi) + \frac{1}{\sigma}R_\phi(s,\xi)\right),
\end{equation}
where
\begin{equation}\label{eq:hessian_quantities}
\begin{aligned}
H_\theta(s,a)&\coloneqq\nabla_\phi^2 \LCLIP(s,a),\\W_\theta(s,a)&\coloneqq\indicator{\mathcal U_\theta}(s,a)r_\theta(a|s)\hat A(s,a),\\
    G_\phi(s,\xi)&\coloneqq J_{\mu_\phi}(s)^\top(\xi\xi^\top-I)J_{\mu_\phi}(s),\\
    R_\phi(s,\xi)&\coloneqq\sum_{k=1}^n\xi^{(k)}\nabla^2_\phi\mu_\phi^{(k)}(s).
\end{aligned}
\end{equation}
Since $\xi$ is a bijection on $\cS\times\cA$, it is possible to compute the expectation interchangeably with respect to $\xi$ and $a$ by means of the pushforward measure it induces on $\nu^T_{\pi_\theta}$.
Following the above decomposition we obtain the following theorem.
\begin{thm}\label{thm:lambda_max_bound_squared}
Assume that there exist constants $\chi_\theta, C_\theta>0$ such that
\begin{equation}\label{eq:chi_C_theta}\begin{aligned}
    &\E_{s,a}
    \left[W_\theta^2\|G_\phi\|_{\mathrm{op}}^2\right]
    \geq\chi_\theta \E_{s,\xi}\left[\|G_\phi\|_{\mathrm{op}}^2\right]\\
    &\E_{s,a}\left[W_\theta^2\|R_\phi\|_{\mathrm{op}}^2\right]
    \leq C_\theta.
    \end{aligned}
\end{equation}
Then,
\begin{equation}\label{eq:hessian_sharpness_square_bound}
    \begin{aligned}
    &\E_{s,a}
    \left[
        \rho
        \left(
            H_\theta^2
        \right)
    \right]
    \geq\chi_\theta
    \frac{\sigma_{old}^4}{\sigma^8}
    \E_s
    \left[
        \rho
        \left(
            J_{\mu_\phi}J_{\mu_\phi}^\top
        \right)^2
    \right]
    -
    \frac{C_\theta}{\sigma^2},
    \end{aligned}
\end{equation}
where $\rho(\cdot) = \lambda_{\max}(\cdot)$ denotes the spectral radius, or, equivalently, the largest eigenvalue of the square matrix in its argument. 
\end{thm}
\paragraph{Proof sketch.} Using the inequality $\|A+B\|^2\geq\frac12\|A\|^2-\|B\|^2$ and Assumption~\eqref{eq:chi_C_theta}, we bound $\E_{s,a}[\rho(H_\theta(s,a)^2)]$ with  $\chi_\theta\E_{s,\xi}[\|G_\phi(s,\xi)\|_{\mathrm{op}}^2]/(2\sigma^4)$. Applying the Rayleigh quotient inequality to $\|G(s,\xi)\|_{\mathrm{op}}$ using the unit eigenvector $u_1$ of $J_{\mu_\phi}(s)J_{\mu_\phi}(s)^\top$ associated with the largest eigenvalue $\lambda_1$ yields $\|G_\phi(s,\xi)\|_{\mathrm{op}}^2\geq\lambda_1^2((u_1^\top\xi)^2-1)^2$. 
By studying the (normal) distribution of $u_1^\top\xi$ we can show that $\E_\xi[(u_1^\top\xi)^2-1)^2|s]\geq2\sigma_{old}^4/\sigma^4$. The result is shown by averaging over $s$ and applying the tower rule. The complete proof is provided in Appendix~\ref{app:proof_lambda_max_thm}.

\negspace Theorem~\ref{thm:lambda_max_bound_squared} yields the following corollary.
\begin{cor}\label{cor:largest_eigenvalue_sqrt_bound}
    Under the assumptions of Theorem~\ref{thm:lambda_max_bound_squared}, if there exists a constant $ M_\theta>0$ such that
    \small{\begin{equation}\label{eq:M_theta}
    \begin{aligned}
        0 <\, &\E_{s,a}\left[\rho\left( H_\theta(s,a)^4\right)\right]
        \leq\, M_\theta
        \E_{s,a}\left[\rho\left( H_\theta(s,a)^2\right)\right]^2 <\infty,
        \end{aligned}
    \end{equation}}%
    then, it holds that
    \begin{equation}
        \E_{s,a}\left[\max_i|\lambda_i\left( H_\theta(s,a)\right)|\right]\geq \sqrt{\frac{L_\theta}{2}}\frac{1}{4M_\theta},
    \end{equation}
    where $\lambda_i(\cdot)$ denotes the $i$-th largest eigenvalue and
    \begin{equation}
         L_\theta \coloneqq \left[\chi_\theta\frac{\sigma_{old}^4}{\sigma^8}\E_s\left[\rho\left(J_{\mu_\phi}(s)J_{\mu_\phi}(s)^\top\right)^2\right]
    - \frac{C_\theta}{\sigma^2}\right]^+,
    \end{equation}
    with $[\cdot]^+$ denoting the positive part.
\end{cor}
\paragraph{Proof sketch.} Assumption~\eqref{eq:M_theta} is a kurtosis-like assumption ensuring that it is possible to take the square root inside the argument of $\E[\cdot]$ and on the RHS of~\eqref{eq:hessian_sharpness_square_bound}, provided that everything is positive. It is a direct application of Lemma~\ref{lemma:kurtosis}, found in Appendix~\ref{app:sec_lemma_kurtosis}.

Corollary~\ref{cor:largest_eigenvalue_sqrt_bound} shows that if $\sigma\approx\sigma_{old}$, and $C_\theta=o(\sigma^{-2}),\ \liminf_{\sigma\to0}\chi_\theta>0$, then the Hessian sharpness is expected to scale as $\sigma^{-2}$. This creates problems with optimization and is a source of asymptotic collapse in stochastic gradient descent, if the learning rate exceeds $2/\E_{s,a}\left[\max_i|\lambda_i( H_\theta(s,a))|\right]$~\citep{cohen2021gradient}.\negspace
\subsection{Collapse of the Hessian approximate rank}\negspace
While the previous result is important in defining a suitable learning rate for stable learning dynamics and can be used to explain the asymptotic collapse on PPO~\citep{Dohare2024},~\cite{lewandowski2024curvature} has shown that the sharpness itself is not the sole factor affecting LoP, and instead proposes using the effective rank $\mathrm{srank}^{(\delta)}(\nabla^2\cL)\coloneqq\argmin_k \left\{
    \frac{\sum_{i=1}^k\sigma_i(\nabla^2\cL)}{\sum_{i=1}^d\sigma_i(\nabla^2\cL)}>1-\delta
    \right\}$
for some value $0<\delta<1$, where $\sigma_i(\cdot)$ denotes the $i$-th singular value. 
The purpose of the following theorem is to show that in PPO the collapse of $\sigma\to0$ causes the reduction of the rank of the empirical matrix to, at most, $n\times N$, where $n$ and $N$ respectively denote the action dimensionality the batch size.
\negspace\begin{thm}\label{thm:hessian_effective_rank_collapse}
    Let $B=\{(s_i, a_i)\}_{i=1}^N$ be a batch of $N$ state-action pairs and consider the decomposition in~\eqref{eq:hessian_implicit}.
    If 
    $
        M_1\coloneqq \sum_{i=1}^N W_\theta(s_i,a_i)G_\phi(s_i,\xi_i)\neq 0,
    $
then the matrix
\begin{equation}
    \bar H_\theta(B) \coloneqq  \,\nabla^2_\phi\left(\frac1N\sum_{i=1}^N\LCLIP(s_i,a_i)\right)
\end{equation}
has at least one and at most $n\times N$ eigenvalues that scale as $1/\sigma^2$, where $n$ is the action dimensionality.\\
Furthermore, for every $0<\delta<1$,
\begin{equation}\label{}
    \limsup_{\sigma\to0}\mathrm{srank}^{(\delta)}(\bar H_\theta(B))\leq \rk{M_1}\leq nN.
\end{equation}
\end{thm}
\paragraph{Proof sketch} The matrix $\bar H_\theta$ can be decomposed into two matrices $M_1$ and $M_2$ so that $\bar H_\theta(B) = \frac{1}{N\sigma^2}(M_1+\sigma M_2)$, implying that one of them scales with $1/\sigma^2$ while the other scales with $1/\sigma$.
By the continuity of the eigenvalues, this also applies to the eigenvalues of $\bar H_\theta$, 
and since $M_1$ has at most $n\times N$ non-zero eigenvalues, by the definition of $G_\phi$ we get that at most $n\times N$ eigenvalues of $\bar H $ scale with $1/\sigma^2$.
The same relationship can be applied to the singular values, which are used in the $\mathrm{srank}$ definition. 
Since the sum of all the singular values converges to the sum of the singular values of $M_1$ as $\sigma\to0$, we obtain the desired result.
A complete proof is provided in Appendix~\ref{app:proof_hess_eff_rank}.\negspace
\subsection{Results interpretation and consequences}\negspace
Given the relationship between step size and sharpness, Theorem~\ref{thm:minibatch_gradient_norm} and Corollary~\ref{cor:largest_eigenvalue_sqrt_bound} would suggest the following heuristic scaling of the learning rate:
$
    \eta_t = \eta_0\frac{\sigma^2}{\|\bar g(B)\|},
$
with $\bar g(B)$ given by~\eqref{eq:bar_g_theta} for SGD, while the normalization by $\|\bar g(B)\|$ would not be needed in Adam or RMSProp, as they naturally provide scale invariance. 
Interestingly, this is the same scaling proposed in Safe Policy Gradients~\citep{papini2020learningspeedstability}. 
However, our results formally motivate scaling by $\|\bar g(B)\|$, which was originally proposed simply as a heuristic. We also extend the results to the complex loss structure of PPO. An additional degree of novelty in our derivation is the use of a  fundamentally different approach, based on geometric attributes of the PPO loss landscape.

\negspace
Moreover, since Theorem~\ref{thm:minibatch_gradient_norm} implies that the gradient norm diverges as $\sigma\to0$, this motivates why setting $\beta_1=\beta_2$ in Adam, instead of the conventional values of $\beta_1=0.99$ and $\beta_2=0.999$, reduces the step-size~\citep{Dohare2024}. 
In fact, the ratio between $\beta_1$ and $\beta_2$ also determines how much weight is given to the latest gradient compared to the previous ones. 
Since the gradient norm diverges, this implies that increasing $\beta_1/\beta_2$ makes the update larger in magnitude. Corollary~\ref{cor:largest_eigenvalue_sqrt_bound} complements these considerations, motivating why reducing the step size should mitigate, but not eliminate, the asymptotic collapse. 
In fact, smaller steps are preferable under increasing sharpness conditions. 
However, step-size reduction is not directly related to the sharpness itself and does not prevent collapse. 
Moreover, our results explain the success of methods that normalize updates such as Adam or RMSProp. 
Indeed, these approaches avoid effective step size blow-up during optimization, provided that $\sigma$ collapses smoothly enough.
Finally, our work motivates the findings of~\cite{andrychowicz2021what} that, when using RMSProp and changing the momentum coefficient from 0 to 0.9, the learning rate needs to be reduced. 
If the gradient norms were to go to 0, which normally occurs during optimization, a learning rate increase should instead have been necessary. 
In conclusion, our results provide insights into the learning dynamics of PPO, formally explaining previous heuristics.
\begin{rmrk}
    The assumptions used in Theorems~\ref{thm:minibatch_gradient_norm} and~\ref{thm:minibatch_gradient_norm} are relatively weak, as they only require integrability. In fact, if the expectations in~\eqref{eq:K_theta},~\eqref{eq:k_theta}, ~\eqref{eq:chi_C_theta} and~\eqref{eq:M_theta} exist and are finite, then all the constants exist and the hypotheses are well-posed.
\end{rmrk}\negspace
\begin{rmrk}
    All the results are taken with respect to $\phi$, but since $\nabla^2_\phi[\cdot]$ is a principal submatrix of $\nabla^2_\theta[\cdot]$, which is symmetric, and $\nabla_\phi[\cdot]$ is a subvector of  $\nabla_\theta[\cdot]$ our results also apply to the gradient and sharpness taken with respect to $\theta$.
\end{rmrk}\negspace
\section{EXPERIMENTS}\negspace
In this section, we validate the hypotheses and consequences of Theorems~\ref{thm:minibatch_gradient_norm} and~\ref{thm:lambda_max_bound_squared}. We also perform numerical simulations on different RL benchmarks and real-world deployment of the trained policies on a real Franka Emika Panda robot to further support our results.
Experiments in simulation are performed on a single NVIDIA L40S GPU.
Full implementation details are included in Appendix~\ref{app:implementation_details}.
\negspace
\subsection{Theoretical validation}\label{sec:theoretical_validation}\negspace
\begin{figure*}[t]
    \centering    \includegraphics[width=\linewidth]{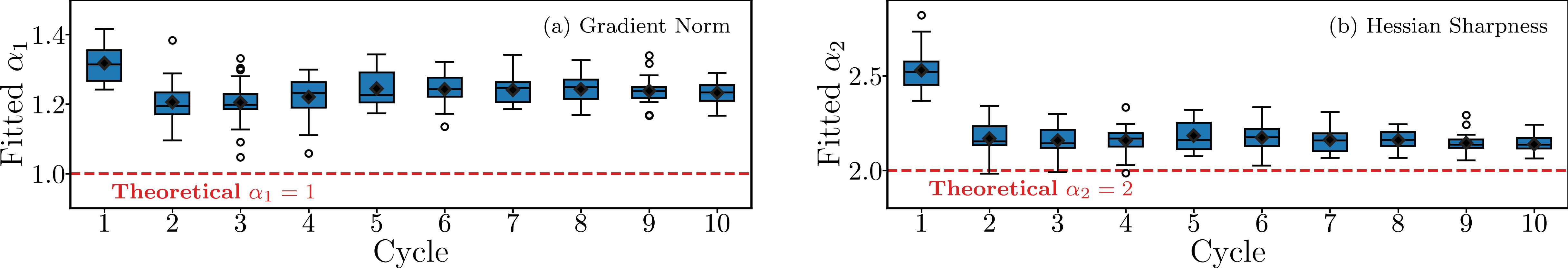}

    \caption{Results of fitting $\alpha_1$ and $\alpha_2$ for the gradient norm and Hessian sharpness.}
    \label{fig:exponents}
\end{figure*}
\begin{figure*}[t]
    \centering
\includegraphics[width=\linewidth]{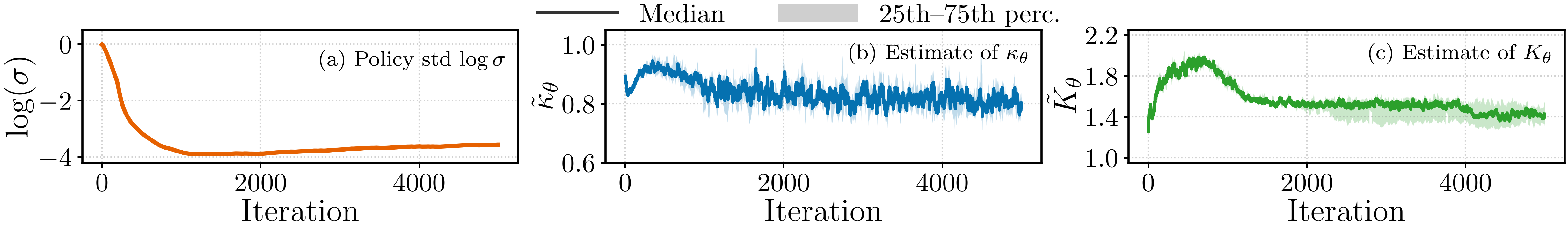}
    \caption{Empirical measurements of $\tilde \kappa_\theta$ and $\tilde K_\theta$, smoothed over a window of 30 iterations.}
    \label{fig:K_theta}
\end{figure*}
The goals of this validation are to ensure that the hypotheses laid out in  Section~\ref{sec:optimization_landscape} are well-posed and, furthermore, to verify the consequences of the theorems.
\paragraph{Experimental setup} 
As there is no method known to the authors to efficiently estimate $\chi_\theta$ and $C_\theta$ in~\eqref{eq:chi_C_theta}, we only compute estimates for the constants of Theorem~\ref{thm:minibatch_gradient_norm}.
To do so, we train policies for the standard \code{Isaac-Reach-Franka-v0} IsaacLab\footnote{\url{https://github.com/isaac-sim/IsaacLab}.} environment, using the RSL-RL\footnote{\url{https://github.com/leggedrobotics/rsl_rl}.} implementation of PPO. In this task, the policy controls a Franka Emika Panda manipulator robot, and its goal is to move the end-effector to a target position and orientation.

Let $p$ denote the number of batches per iteration and $N$ be the size of each batch $B_j$, $j\in{1,\ldots,p}$. Let
$
    \tilde \kappa_\theta^{\textrm{(batch)}}(B_j) \coloneqq \frac{\frac1N\sum_{i=1}^m \nabla_\phi \LCLIP(s_i,a_i) }{\frac1N\sum_{i=1}^m(\nabla_\phi\log\pi_\theta(s_i,a_i))^2}
$
for each batch. Then, at timestep $t$, $\kappa_\theta$ is estimated by $
    \tilde \kappa_\theta^{(t)} = \frac1p\sum_{j=1}^p \tilde \kappa_\theta^{\textrm{(batch)}}(B_j).
$
Analogously, $K_\theta$ is estimated as
$
        \tilde K_\theta^{(t)} \coloneqq \frac{\frac1p\sum_{j=1}^p\bar g_\theta(B_j)^4}{(\frac1p\sum_{j=1}^p\bar g_\theta(B_j)^2)^2}.
$
Then, to validate the impact of $\sigma$ over the gradient norm and the largest Hessian eigenvalue, we reset $\sigma$ every 500 iterations to its initial value and report how it affects these quantities. This periodic resetting allows us to establish a causal relationship between $\sigma$ and these metrics, since it changes quickly compared to the other values. The analysis also includes a loglinear fit of each 500-iteration window according to
\begin{equation}\label{eq:log_regression_estimates}
\begin{aligned}
    \log  \tilde g_\theta^{(t)} &= -\alpha_1 \log \sigma + \beta_1,\\
    \log \lambda_{\max}^{(t)}&= -\alpha_2 \log \sigma + \beta_2,
\end{aligned}
\end{equation}
where $\tilde g_\theta^{(t)}$ and $\lambda_{\max}^{(t)}$ are the average of the gradient norm and the Hessian spectral norm across the gradient descent batches performed at iteration $t$. 
We report the exponents $\alpha_1$ and $\alpha_2$, to compare them with the theoretical bounds. Theorems~\ref{thm:minibatch_gradient_norm} and~\ref{thm:lambda_max_bound_squared} imply that $\alpha_1\geq1$ and $\alpha_2\geq2$.
The training is performed with 4096 parallel environments, for 5000 training iterations, and the results are averaged over 30 seeds. One iteration consists of 24 timesteps, yielding $24\times4096=98304$ timesteps per iteration.
The agents are trained using the Adam optimizer with 8 epochs of 4 minibatches each, yielding $p=32$ gradient descent steps per iteration. 
\negspace
\paragraph{Results}
The empirical estimates of $\kappa_\theta$ and $K_\theta$ are shown in Figure~\ref{fig:K_theta}.
Consistently with the theoretical assumptions, they neither diverge nor collapse to 0 during training.
Moreover, training \code{Isaac-Reach-Franka-v0} with $\sigma$ resets every 500 iterations and fitting the log-linear model from~\eqref{eq:log_regression_estimates} over each of these 500-iteration cycles results in the $\alpha_1$ and $\alpha_2$ estimates shown in Figure~\ref{fig:exponents}. These findings are consistent with our theoretical results, as we get $\alpha_1$ slightly above 1 and $\alpha_2$ hovering around 2. The mean of the $\text{R}^2$ of $\alpha_1$ is 0.6409, while for $\alpha_2$ it is 0.8728.
More results are shown in Appendix~\ref{app:additional results}.

\negspace
\subsection{The unreported role of $\sigma$ in a common benchmark}\negspace
\begin{figure*}[t]
\begin{subfigure}[t]{0.49\textwidth}
        \centering
\includegraphics[width=\linewidth]{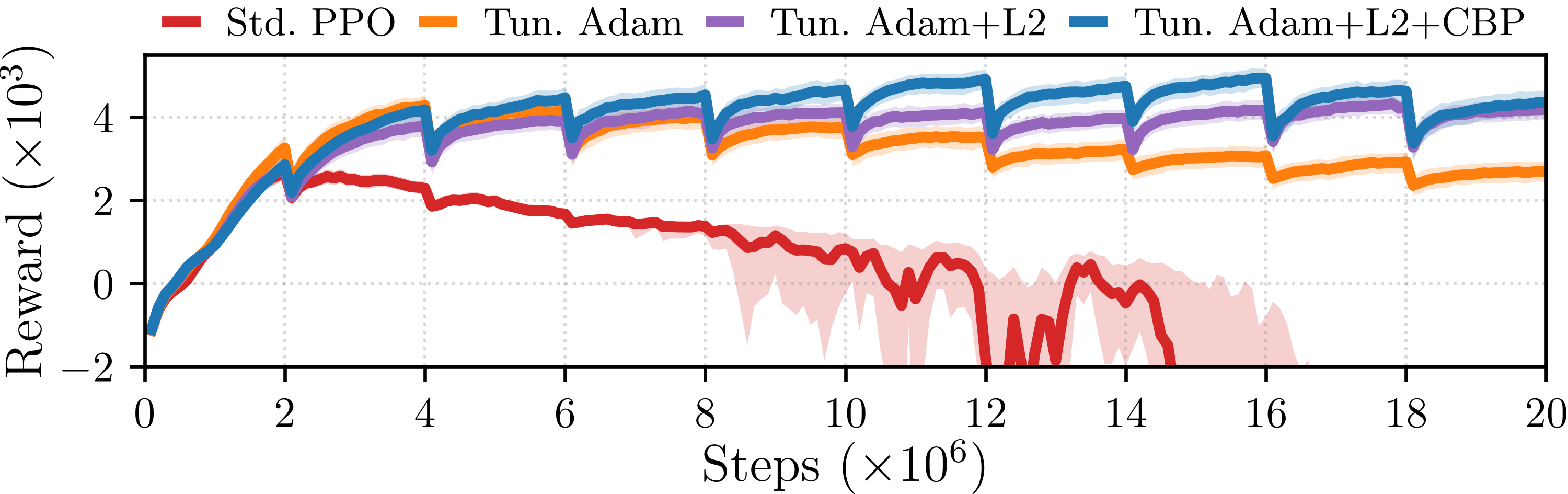}
        \caption{}\label{fig:rewards_fig3}
\end{subfigure}
\begin{subfigure}[t]{0.49\textwidth}
        \centering
\includegraphics[width=\linewidth]{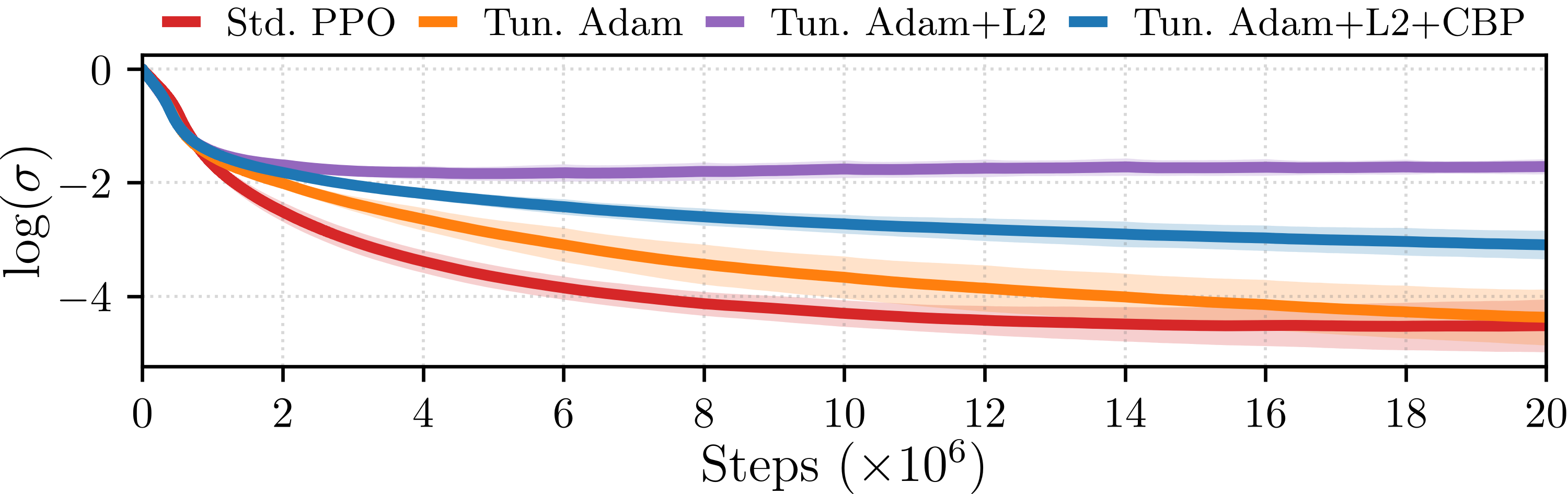}
        \caption{}\label{fig:sigma_sa_fig3}
\end{subfigure}
    \caption{\textbf{(a):} Reward results in Slippery Ant and \textbf{(b):} evolution of $\log\sigma$ across the different methods. Reward collapse is associated with $\sigma$'s collapse and the studied techniques strongly affect its evolution.}
    \label{fig:sigma_sa}
\end{figure*}
\begin{figure*}[t]
    \centering
    \begin{subfigure}[t]{0.49\textwidth}
        \centering
\includegraphics[width=\linewidth]{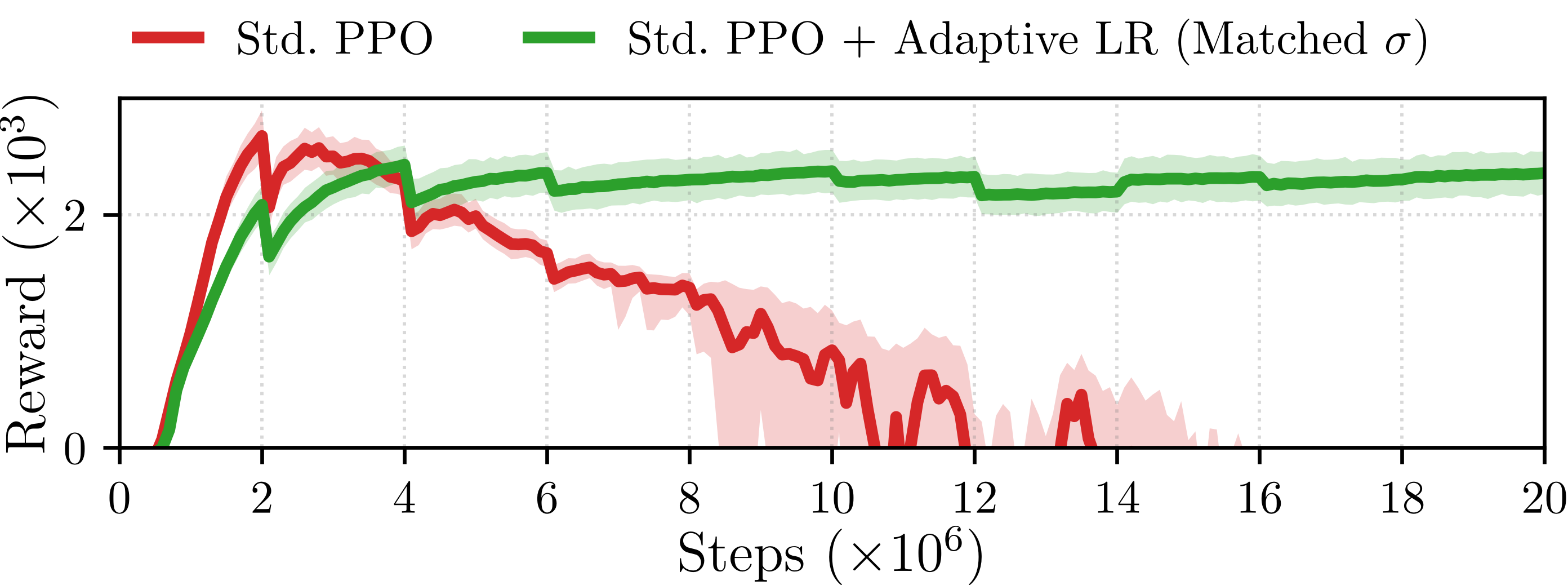}
        \caption{}
        \label{fig:adaptive_lr_sa}
    \end{subfigure}
    \hfill
    \begin{subfigure}[t]{0.49\textwidth}
        \centering
\includegraphics[width=\linewidth]{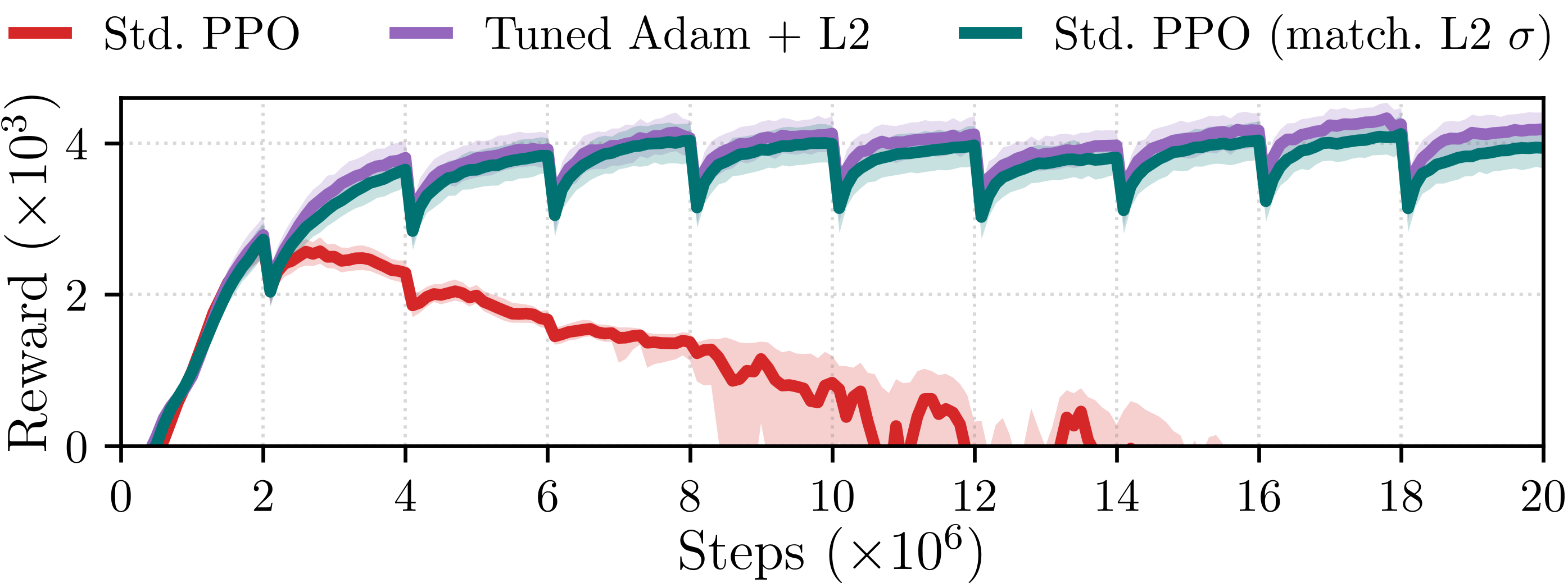}
        \caption{}
        \label{fig:std_l2_sa}
    \end{subfigure}
    \caption{Training progress in Slippery Ant. \textbf{(a):} Reward of Standard PPO and Standard PPO with adaptive learning rate. \textbf{(b):} Reward of Standard PPO with the same evolution of $\sigma$ as Tuned Adam + L2.}
    \label{fig:combined_slippery_ant_all}
\end{figure*}

In this experiment, we explore whether $\sigma$ plays a role in the effectiveness of state-of-the-art LoP mitigation techniques, and if it can explain asymptotic PPO collapse. We do so using the Slippery Ant environment, which was popularized by~\cite{dohare2022continual, Dohare2024} and is generally acknowledged as a benchmark for LoP in PPO~\citep{lazar2026revisiting,akgul2025overcoming,klein2026plasticitylossdeepreinforcement}. 
It is based on the MuJoCo Ant environment\footnote{\url{https://github.com/google-deepmind/mujoco}}, with the twist that every 2M timesteps the friction coefficient between the robot and the ground is randomized within the range [0.02, 2.00], so that the environment is non-stationary. 
The main result reported by \cite{Dohare2024} is that regular PPO collapses in this setting, whereas tuning Adam such that $\beta_1=\beta_2=0.99$ delays the collapse,
and adding L2 regularization prevents it.
However, it underperforms the Continual Backpropagation (CBP) technique~\cite{Dohare2024}, in which a low percentage of weights with low utilization scores are randomly reinitialized. 
\paragraph{Experimental setup.} We build upon \cite{Dohare2024} by comparing our results in the same setting\footnote{Based on the official repository \url{https://github.com/shibhansh/loss-of-plasticity}.}. To emphasize the hidden and unreported effect of $\sigma$ across the different baselines, we first replicate the results while recording the evolution of $\sigma$ across 100 seeds. Given that this environment does not employ isotropic $\sigma$, but rather a diagonal matrix with $n=8$ entries, the value we report as $\sigma$ is the average standard deviation $\frac 1n\sum_{i=1}^n\sigma_i$,
where $\sigma_i$ is the $i$-th entry of the diagonal variance-covariance matrix.

In the original paper, the comparison between the different methods was performed on the stochastic policies. As we found significantly different levels of $\sigma$, we compare the performance of the deterministic final policy obtained with the two most promising methods, L2 and L2+CBP, to distinguish between the effectiveness of each. 
This is motivated by the fact that RL policies are often deployed deterministically in the target environment as standard assumptions ensure the existence of an optimal deterministic policy (cfr.~\cite{hernandezlerma1996discretetime}, Theorem 3.2.1).
We then conduct two more experiments, to further validate the centrality of the standard deviation in the context of LoP:
\begin{enumerate}[noitemsep,leftmargin=*,topsep=0pt]
    \item Schedule $\sigma$ to evolve as in Standard PPO, so that it is not learned. Then, progressively anneal the learning rate with $\eta_t = \eta_0\sigma_t^2$. This test is performed primarily to assess whether the collapse of PPO can be avoided.\label{item:experiment1}
    \item Schedule $\sigma$ to evolve in the same way as the baseline trained with tuned Adam + L2 regularization, and use this schedule to train Standard PPO. This experiment aims to evaluate what portion of the improvements on LoP can be explained simply by the implicit regularization of $\sigma$ and how it evolves.\label{item:experiment2}
\end{enumerate}
\begin{figure}[t]
    \centering
\includegraphics[width=\linewidth]{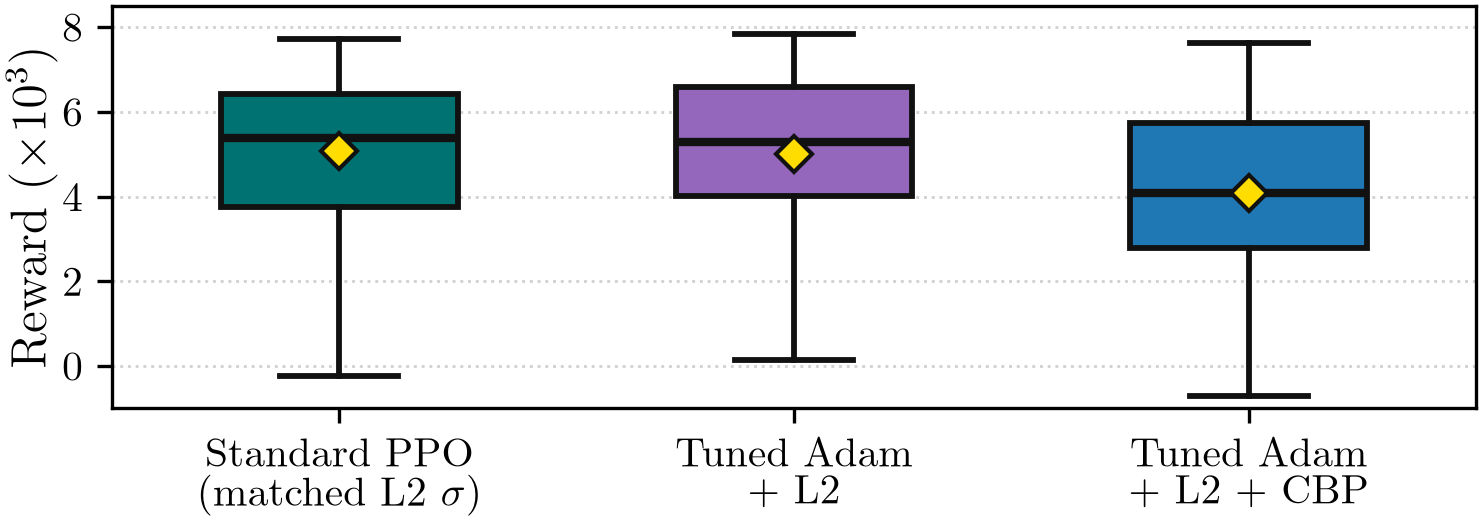}
    \caption{Reward of final deterministic policies for L2, CBP, and Standard PPO with matched $\sigma$ (100 seeds).}
    \label{fig:deterministic_box_sa}
\end{figure}\negspace
\paragraph{Results} Figure~\ref{fig:sigma_sa} demonstrates how the analyzed techniques implicitly affect $\sigma$ and how its collapse is associated with the reward collapse.
\cite{Dohare2024} shows that Standard PPO collapses after 4 million iterations,  tuning Adam is only able to delay this collapse, 
and L2 and L2+CBP show no collapse. We include these results in Figure~\ref{fig:rewards_fig3}. Interestingly, as shown in Figure~\ref{fig:sigma_sa_fig3}, the latter two techniques are also the ones showing stronger $\sigma$ regularization, supporting our theoretical insights.
Figure~\ref{fig:adaptive_lr_sa} contains the results of Experiment~\ref{item:experiment1}. Reducing the learning rate, according to the sharpness increase proven in Theorem~\ref{thm:lambda_max_bound_squared}, is sufficient for avoiding the asymptotic collapse.
However, only using an adaptive learning rate in Standard PPO does not recover the ability to adapt to new environments. Indeed, while other methods achieve rewards of approximately 4000, this approach stagnates at rewards slightly above 2000.
Figure~\ref{fig:std_l2_sa} shows that simply matching the evolution of $\sigma$ to the L2 regularization baseline yields comparable results. 
Moreover, Figure~\ref{fig:deterministic_box_sa} shows that the evaluation of the two deterministic policies obtained is completely equivalent. 
It also highlights that $\sigma$ is a major confounding factor, as CBP significantly underperforms the other two methods, in contrast with the findings of~\cite{Dohare2024}, achieving an average reward that is 20\% lower.
Indeed, these results indicate that many state-of-the-art techniques that try to mitigate LoP by modifying the reward structure or performing operations on the neural networks' weights are indistinguishable from approaches that simply focus on $\sigma$. 
Their effectiveness can be fully explained by the standard deviation alone.
\negspace
\subsection{Task perturbation with real world deployment}\negspace
In this section we explore the role of $\sigma$ in adapting a policy to a perturbed task. To do this, we train on the \code{Isaac-Reach-Franka-v0} and perturb it by locking one joint after a pretraining phase. To further validate the results, we deploy the learned policies on a real 7-DoF Franka Emika Panda robotic manipulator.
\negspace
\subsubsection{Results in simulation}\negspace
\begin{figure}[t]
    \centering
\includegraphics[width=\linewidth]{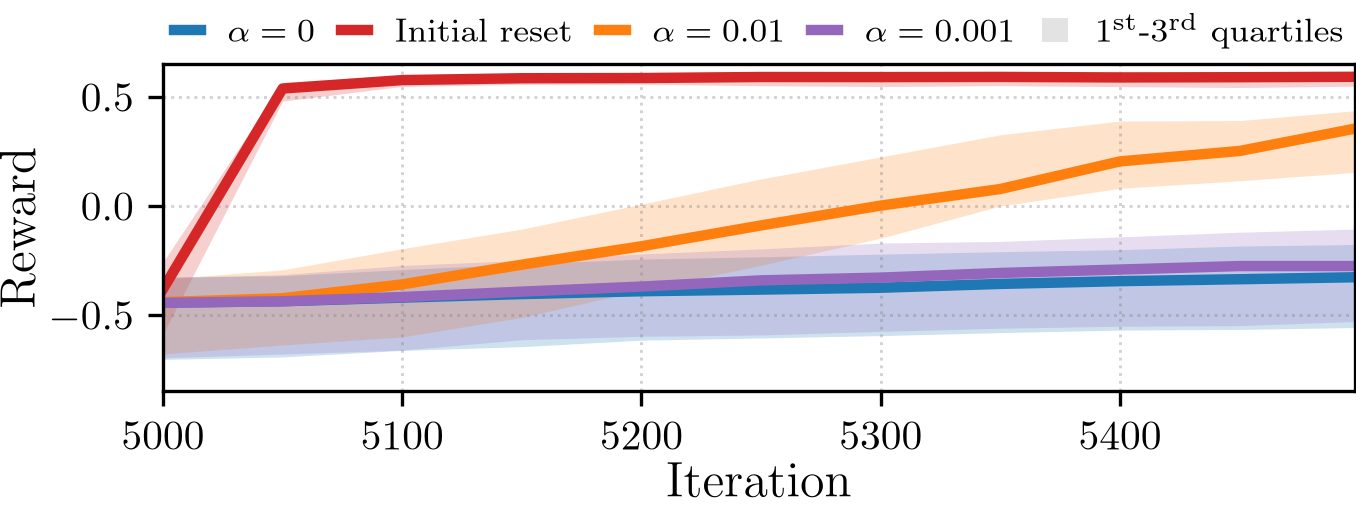}
    \caption{Deterministic policies rewards curve trained in \code{Isaac-Reach-Franka-v0} with the base joint locked.}
    \label{fig:lj_curves}
\end{figure}

\paragraph{Experimental setup} In this experiment, we train 
the policy
on \code{Isaac-Reach-Franka-v0} with 4096 parallel environments over 30 seeds for 5000 iterations. 
For pretraining, we use the policies trained with periodic $\sigma$ resets,
as we found that they show smaller variance across seeds and better performance on the main task. 
However, this choice does not affect our results, as we only aim to show that $\sigma$ has a strong impact over the ability of a pretrained policy to learn a new task, regardless of the pretraining process.
After this initial pretraining phase, the task is perturbed by locking the first joint, so that the policy can no longer rely on it to reach its target end-effector position, and we continue  training for 500 iterations. 
To assess the relevance of $\sigma$ to task adaptation, we compare different training strategies: we train with different values of the entropy regularization coefficient $\alpha\in\{0,0.001,0.01\}$ and with a single initial restart of $\sigma$ to its initial value of 1.0.\negspace
\paragraph{Results}Figure~\ref{fig:lj_curves} compares the reward curves of the deterministic policies across the 30 seeds.
These results clearly indicate that there is a very strong advantage in increasing $\sigma$ either by stronger entropy regularization or with an initial reset under non-stationarity.
\negspace
\subsubsection{Deployment on the real robot}\negspace
\begin{figure}[t]
    \centering\includegraphics[width=\linewidth]{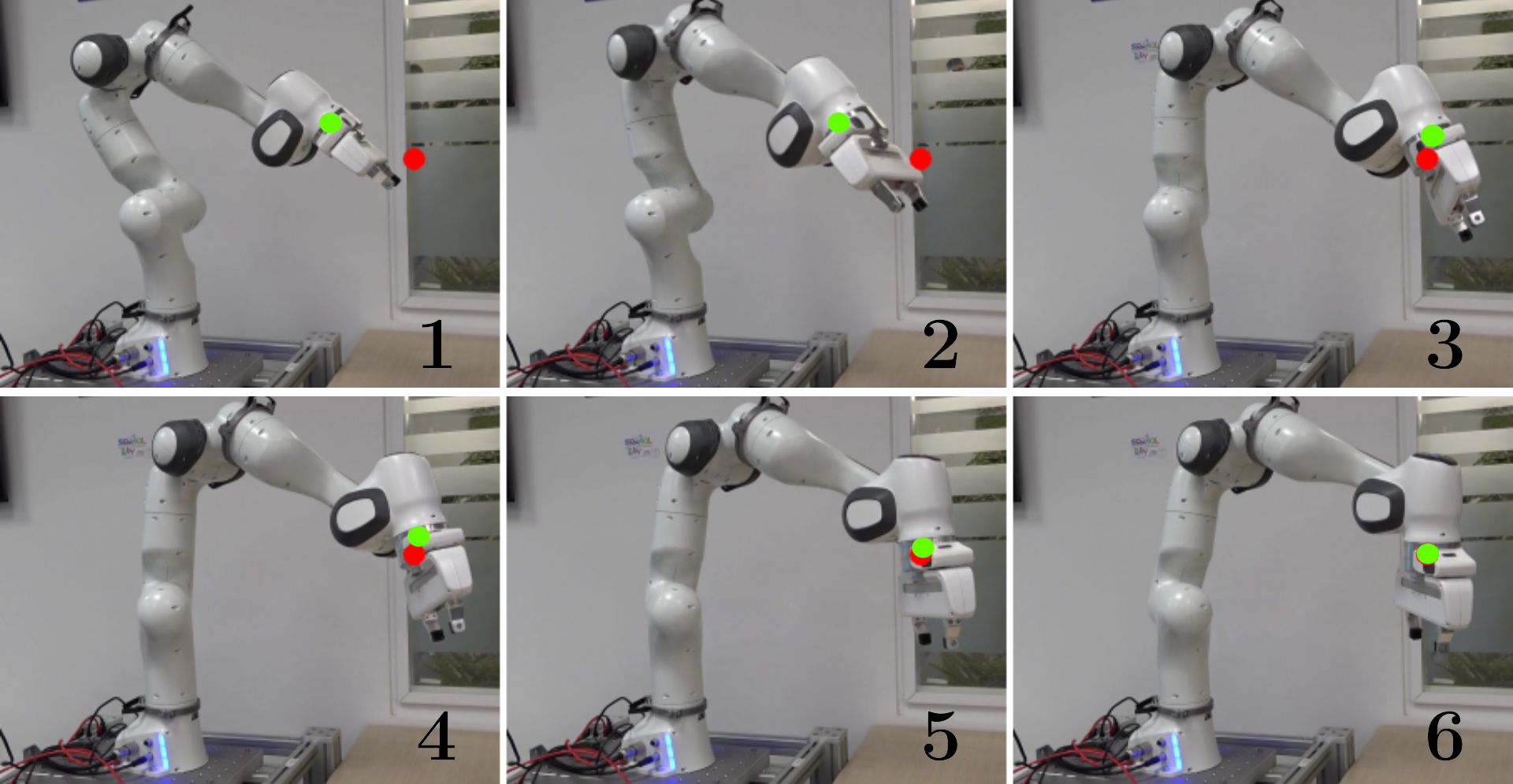}
    \caption{Qualitative deployment on the real Franka Emika Robot of a policy trained on the base task \code{Isaac-Reach-Franka-v0}. The red and green dots indicate the target and current position, respectively.}
    \label{fig:panda_qualitative}
\end{figure}
\begin{figure}[t]
    \centering
\includegraphics[width=\linewidth]{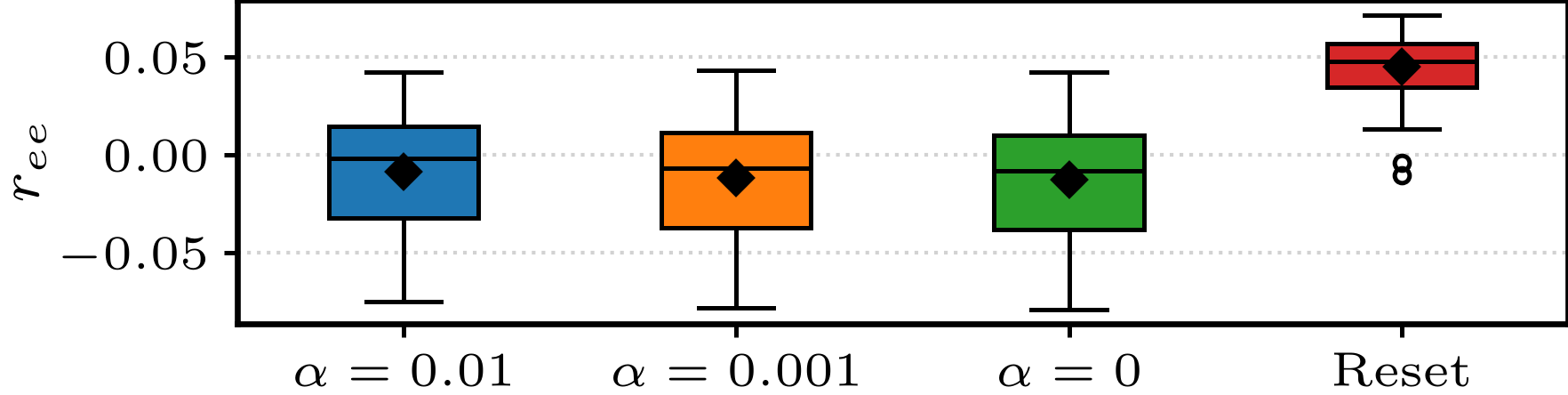}
    \caption{Reward $r_{ee}$ achieved in the real world deployment of the joint-locking task.}
    \label{fig:boxplots_real_world}
\end{figure}
\begin{table}[t]
\centering
\small
\setlength{\tabcolsep}{6pt}
\begin{tabular}{|l | c c|c c|}
\hline
\textbf{Method} & \textbf{Mean} $r_{ee}$ & \textbf{Std}  & \textbf{$p$-value} & \textbf{Sig.} \\
\hline
Reset          & \textbf{0.0449}   & \textbf{0.0178}   & --            & --  \\
$\alpha=0.01$  & \underline{-0.0086} & \underline{0.0301} & $< 10^{-16}$  & *** \\
$\alpha=0.001$ & -0.0118           & 0.0309            & $< 10^{-16}$  & *** \\
$\alpha=0$     & -0.0130           & 0.0309            & $< 10^{-17}$  & *** \\
\hline
\end{tabular}
\caption{Welch $t$-tests in the locked-joint task comparing each method against initial $\sigma$ reset approach.}
\label{tab:weighted_error_ttest_locked_joint}
\end{table}
\paragraph{Experimental setup} The real-world setup as well as the qualitative deployment of one pretrained policy are shown in Figure~\ref{fig:panda_qualitative}.
To compare the different methods, the 30 seeds are ranked based on their overall reward in simulation, and for each method we choose the seed achieving the median reward after 100 iterations. 
In fact, as shown in Figure~\ref{fig:lj_curves}, after 100 iterations the reset of $\sigma$ already solves the task.
Then, the deployment is performed with this seed only, sampling 50 different target points for each method. 
Notably, at the beginning of the test of each policy, the arm is reset to the default joint configuration $[0.0, -0.57, 0.0, -2.81, 0.0, 3.04, 0.74]$ and the same target points are used to compare the different policies.
No reset occurs between target points, so  the previous position is used as a starting point to reach the next one.
We compare the final position and orientation errors, as these metrics are not strongly affected by the sim-to-real gap.
Let $p_{ee,t}$ and $q_{ee,t}$ denote the current end-effector position and quaternion, while $p^{*}_{ee,t}$ and $q^{*}_{ee,t}$ denote the target ones. Then, given
\begin{equation}
    \begin{aligned}
        \text{err}_\text{pos} &=\|p_{ee}-p^{*}_{ee}\|_2 \\
        \text{err}_\text{ori} &=\left\|\log(\Phi(q_{ee}) \Phi^{-1}(q^{*}_{ee}))^\vee\right\|_2 \\
        \text{err}_\text{pos}^\text{(fine)}&=1 - \tanh\left(
        \frac{\text{err}_\text{pos}}{0.1}\right)
    \end{aligned}
\end{equation}
the scores we compare are obtained by
\begin{equation}\label{eq:weighted reward}
    r_{ee} = -0.2\, \text{err}_\text{pos} + 0.1 \text{err}_\text{pos}^{\text{(fine)}}
    -0.1\,\text{err}_\text{ori}.
\end{equation}
Moreover,
$\Phi(\cdot)\colon \mathbb{H}\to SO(3)$ maps the quaternions to the $3\times 3$ corresponding orthogonal matrix, $\log\colon SO(3)\to\frak{so}(3)$ is the matrix logarithm and $(\cdot)^\vee\colon\frak{so}(3)\to\R^3$ maps the elements of the Lie algebra to their three-dimensional representation.
The structure and the coefficients of~\eqref{eq:weighted reward} are motivated by the reward used at training time, excluding penalties for abrupt changes in position and velocity. To compare the results and ensure statistical significance, we perform the Welch $t$-test on the mean value of $r_{ee}$ between the policies trained with different values of $\alpha$ and the one trained with the reset of $\sigma$, and also report the standard deviation of the reward across points.
\negspace
\paragraph{Results}Figure~\ref{fig:boxplots_real_world} reports the results, while the Welch $t$-tests are shown in Table~\ref{tab:weighted_error_ttest_locked_joint}.
Simulation results successfully translate to real world tests, with a statistically significant higher reward, improved from approximately  -0.01 to 0.045.
Also, the standard deviation of the reward is significantly higher for the baselines compared to the one achieved by resetting $\sigma$.
\negspace
\section{CONCLUSIONS}\negspace
We have shown a novel connection between LoP and exploration with Gaussian policies, with both theoretical and experimental results.
We derived mathematical bounds characterizing the PPO loss landscape, showing how $\sigma\to0$ causes the gradient norm and Hessian sharpness to diverge and the Hessian approximate rank to collapse. We then numerically verified these results and provided experimental evidence, including on a real-world robotic task, that  LoP can be caused by $\sigma\to0$.

Future work will  examine whether policy entropy  affects LoP for other state-of-the-art RL methods and investigate its relationship with catastrophic forgetting.

\section*{Acknowledgments}
This study was carried out within the FAIR - Future Artificial Intelligence Research and received funding from the European Union Next-GenerationEU (PIANO NAZIONALE DI RIPRESA E RESILIENZA (PNRR) – MISSIONE 4 COMPONENTE 2, INVESTIMENTO 1.3 – D.D. 1555 11/10/2022, PE00000013). This manuscript reflects only the authors’ views and opinions, neither the European Union nor the European Commission can be considered responsible for them. E.\ C.\ acknowledges support from COST Action InterCoML (CA24136).

\bibliography{bibliography}

\clearpage
\appendix
\thispagestyle{empty}

\onecolumn

\section{Useful results}
We provide here and explicit derivation of the gradient and the Hessian of $\LCLIP$, as well as a lemma used in the proofs.
\subsection{Gradient of the clipped loss}
Let $\unclipped$ under policy $\theta$ and clipping parameter $\varepsilon$ be the set
\begin{equation}
    \unclipped \coloneqq \{(s,a) \colon\, r_\theta(s,a)<1+\varepsilon, \hat A(s,a) \geq 0\}\cup\{(s,a) \colon\, r_\theta(s,a)>1-\varepsilon, \hat A(s,a) \leq 0\}
\end{equation}
so that we can rewrite $\LCLIP$ as follows:
\begin{equation}\label{eq:lclip_piecewise}
    \LCLIP(s,a) = 
    \begin{cases}
        r_\theta(s,a) \hat A(s,a) \qquad &\text{if } (a,s)\in\unclipped \\
        (1-\varepsilon)\hat A(s,a) \qquad &\text{if } (a,s)\not\in\unclipped \text{ and } \hat A(s,a)<0\\
        (1+\varepsilon)\hat A(s,a) \qquad &\text{if } (a,s)\not\in\unclipped \text{ and } \hat A(s,a)\geq0\\
    \end{cases}.
\end{equation}
The gradient of $\LCLIP$ with respect to $\theta$ is not defined on $\unclippedbrd$ and it is identically equal to 0 outside of $\unclipped$ as the terms $(1\pm\varepsilon)\hat A(s,a)$ do not depend on $\theta$, so that
\begin{equation}\label{eq:grad_unclipped_region}
     \nabla_\theta\LCLIP(s,a)\equiv 0 \qquad \text{for all } (s,a)\in \left(\unclipped\cup\unclippedbrd\right)^c,
\end{equation}
where $(\cdot)^c$ denotes the complement of a set.
As a consequence of~\eqref{eq:lclip_piecewise} and~\eqref{eq:grad_unclipped_region}, we can write
\begin{equation}\label{eq:lclip_grad_no_log}
    \nabla_\theta\LCLIP(s,a)= \indicator{\unclipped}(s,a)\frac{\nabla_\theta\pi_\theta(a|s)}{\pi_{\theta_{old}}(a|s)} \hat A(s,a),\qquad \text{for all } (s,a)\not\in\unclippedbrd
\end{equation}
where $\indicator{\unclipped}(s,a)$ is the indicator function:
\begin{equation}
    \indicator{\unclipped}(s,a) \coloneqq \begin{cases}
        1 \qquad &\text{if }(s,a)\in\unclipped\\
        0 \qquad   &\text{if }(s,a) \not\in\unclipped
    \end{cases}.
\end{equation}
Recall now the log-derivative trick, i.e., for a given function $f\colon \cX \to \R$, it holds that 
\begin{equation}
    \nabla \log f(x) = \frac{\nabla f(x)}{f(x)},
\end{equation}
and rearranging the terms yields
\begin{equation}\label{eq:log_derivative_trick}
    \nabla f(x)  = f(x)\nabla\log f(x).
\end{equation}
By applying~\eqref{eq:log_derivative_trick} to~\eqref{eq:lclip_grad_no_log}, we obtain
\begin{equation}
    \nabla_\theta\LCLIP(s,a)= \indicator{\unclipped}(s,a)r_\theta(s,a)\nabla_\theta\log\pi_\theta(s,a) \hat A(s,a) 
\end{equation}
for all $(s,a)\not\in\unclippedbrd$, where $r_\theta$ is the policy density ratio $r_\theta(s,a)= \pi_\theta(a\,|\,s)/\pi_{\theta_{old}}(a\,|\,s)$

Let us now assume that $\pi_\theta$ is a deep isotropic Gaussian policy, parameterized by $\theta=(\phi,\sigma)$. The computations above also hold for the gradient with respect to $\phi$, so that we obtain
\begin{equation}\label{eq:lclip_log_trick_gradient}
    \nabla_\phi\LCLIP(s,a)= \indicator{\unclipped}(s,a)r_\theta(s,a)\nabla_\phi\log\pi_\theta(s,a) \hat A(s,a) 
\end{equation}
for all $(s,a)\not\in\unclippedbrd$. 
Since
\begin{equation}
    \log \pi_\theta(a|s) = -\frac{n}{2}\log(2\pi\sigma^2)-\frac{1}{2\sigma^2}(a-\mu_\phi(s))^\top(a-\mu_\phi(s)),
\end{equation}
the derivative with respect to the $j$-th weight $\phi_j\in\phi$ becomes
\begin{equation}
    \begin{aligned}
    \frac{\partial}{\partial\phi_j} \log \pi_\theta(a|s)&= -\frac{1}{2\sigma^2} \frac{\partial}{\partial\phi_j}(a-\mu_\phi)^\top(a-\mu_\phi) \\
    &= \frac{1}{\sigma^2} (a-\mu_\phi)^\top \frac{\partial\mu_\phi(s)}{\partial\phi_j} \\
    &= \frac{1}{\sigma} \xi^\top \frac{\partial\mu_\phi(s)}{\partial\phi_j} \\
    &= \frac{1}{\sigma} (J_{\mu_\phi}^\top(s)\xi)_j, 
    \end{aligned}
\end{equation}
where $J_{\mu_\phi}^\top(s)$ denotes the Jacobian of the deterministic policy $\mu_\phi$ computed in the point $s$. 

\noindent Therefore, we obtain the following
\begin{equation}\label{eq:log_policy_grad_final}
    \nabla_\phi\log\pi_\theta(a|s) = \frac 1\sigma J_{\mu_\phi}^\top(s)\xi
\end{equation}
and equation~\eqref{eq:lclip_log_trick_gradient} becomes
\begin{equation}
    \nabla_\theta\LCLIP(s,a)= \indicator{\unclipped}(s,a)r_\theta(s,a) \hat A(s,a)   \frac 1\sigma J_{\mu_\phi}^\top(s)\xi
\end{equation}
for all $(s,a)\not\in\unclippedbrd$.\\

\begin{rmrk}
    Given a state $s$ we always assume that actions $a$ are sampled from $\pi_{\theta_{old}}$, so that $a\sim\pi_{\theta_{old}}(\cdot|s)=\cN(\mu_{\phi_{old}}(s),\sigma_{old}^2 I)$. 
This is motivated by the typical PPO implementation, consisting in a sample collection phase, in which rollouts are collected according to $\pi_{\theta_{old}}$ and then the policy $\theta_{old}$ is updated by performing multiple optimization steps using the gradient of $\LCLIP$.
\end{rmrk}

\subsection{Hessian of the clipped loss}\label{app:hessian_decomposition}

Let $\xi =\xi(s,a)\coloneqq(a-\mu_{\phi}(s))/{\sigma}$, then
\begin{align*}
    \frac{\partial^2}{\partial\phi_i\partial\phi_j}\log\pi_\theta(a|s) &= \frac{1}{\sigma^2}\frac{\partial}{\partial \phi_i}(a-\mu_\phi(s))^\top\frac{\partial\mu_\phi(s)}{\partial\phi_j}\\
    &=\frac{1}{\sigma^2}(a-\mu_\phi(s))^\top\frac{\partial^2\mu_\phi(s)} {\partial\phi_i\partial\phi_j}-\frac{1}{\sigma^2}\frac{\partial\mu_\phi^\top(s)}{\partial\phi_i}\frac{\partial\mu_\phi(s)}{\partial\phi_j}\\
    &=\frac{1}{\sigma}\xi^\top\frac{\partial^2\mu_\phi(s)} {\partial\phi_i\partial\phi_j}-\frac{1}{\sigma^2}\frac{\partial\mu_\phi^\top(s)}{\partial\phi_i}\frac{\partial\mu_\phi(s)}{\partial\phi_j},
\end{align*}
yielding
\begin{equation}\label{eq:log_policy_hessian_final}
    \nabla^2_\phi\log\pi_\theta(a|s) = -\frac{1}{\sigma^2}J_{\mu_\phi}(s)J_{\mu_\phi}^\top(s) +\frac{1}{\sigma}\sum_{k=1}^n\xi^{(k)}\nabla^2_\phi\mu_\phi^{(k)}(s).
\end{equation}
for all $(s,a)\not \in \unclippedbrd$.
Now, let $r_\theta$ be the policy density ratio $r_\theta(s,a)=\pi_\theta(a|s)/\pi_{\theta_{old}}(a|s)$. Then, we get
\begin{align*}
    \nabla^2_\phi r_\theta(s,a) \hat A (s,a) &= r_\theta(s,a)\hat A(s,a)[(\nabla_\phi\log\pi_\theta(a|s))(\nabla_\phi\log\pi_\theta(a|s))^\top + \nabla^2\log \pi_\theta(a|s)]\\
    &=r_\theta(s,a)\hat A(s,a)\Bigg(\frac{1}{\sigma^2}J_\phi^\top(s)\xi\xi^\top J_\phi(s) - \frac{1}{\sigma^2}J_\phi^\top(s)J_\phi(s)\\
    &\qquad\qquad\qquad\qquad+ \frac{1}{\sigma}\sum_{k=1}^n\xi^{(k)}\nabla^2_\phi\mu_\phi^{(k)}(s)\Bigg)\\
    &=r_\theta(s,a)\hat A(s,a)\left(\frac{1}{\sigma^2}J_\phi^\top(s)[\xi\xi^\top -I] J_\phi(s) + \frac{1}{\sigma}\sum_{k=1}^n\xi^{(k)}\nabla^2_\phi\mu_\phi^{(k)}(s)\right),
\end{align*}
where the second equality comes from~\eqref{eq:log_policy_grad_final} and~\eqref{eq:log_policy_hessian_final}.
Finally, since
\begin{equation}
\nabla^2_\phi\LCLIP(a|s) = \indicator{\unclipped} \nabla^2_\phi r_\theta(a|s) \hat A (s,a)   
\end{equation}
we obtain 
\begin{equation}
    \nabla^2_\phi\LCLIP(s,a) 
    = \indicator{\cU_\theta} r_\theta(a|s)\hat A(s,a)\left(\frac{1}{\sigma^2}J_\phi^\top(s)[\xi\xi^\top -I] J_\phi(s) + \frac{1}{\sigma}\sum_{k=1}^n\xi^{(k)}\nabla^2_\phi\mu_\phi^{(k)}(s)\right).
\end{equation}
away from $\unclippedbrd$, where the derivatives would not be defined.

By adopting the notation of~\eqref{eq:hessian_quantities} we get
\begin{equation}
    H_\theta(s,a)= W_\theta(s,a)\left(\frac{1}{\sigma^2}G_\phi(s,\xi) + \frac{1}{\sigma}R_\phi(s,\xi)\right).
\end{equation}

\subsection{A useful lemma}\label{app:sec_lemma_kurtosis}
The following Lemma will be useful in the Theorem proofs.

\begin{lemma}\label{lemma:kurtosis}
    Let $X\geq0$ be a positive random variable, such that there exists $K>0$ for which
    \begin{equation}\label{eq:lemma_bounded_moments_hyp}
        0<\E[X^2]\leq K \E[X]^2 <\infty.
    \end{equation}
    If there exists a value $C\geq0$ for which
    \begin{equation}
        \E[X] \geq C,
    \end{equation}
    then,
    \begin{equation}
        \E[\sqrt{X}] \geq \frac{1}{4K}\sqrt{\frac C2}.
    \end{equation}
\end{lemma}
\begin{proof}
    Since the first and second moment are bounded and $X\geq0$ we can apply the Paley-Zygmund's inequality,
    \begin{equation}
        \P\left( X\geq \eta\, \E [X]\right)
        \geq
        (1-\eta)^2
        \frac{(\E [X])^2}{\E [X^2]}\qquad  \text{for every } \eta\in(0,1).
    \end{equation}
    Then, by choosing $\eta=1/2$, we apply~\eqref{eq:lemma_bounded_moments_hyp} to the previous inequality and get
    \begin{equation}\label{eq:lemma_bounded_moments_ineq_pz}
        \P\left( X\geq \frac12\, \E [X]\right)\geq \frac{1}{4K}.
    \end{equation}
    Now, since $X>0$ we have
    \begin{equation}
        \E[\sqrt X] \geq \E[\indicator{X>\frac12\E[X]}\sqrt X] = \int_{X>\frac12\E[X]}\sqrt X\de \P.
    \end{equation}
    Naturally, over the event $\{X>\frac12\E[X]\}$ it holds that $\sqrt X>\sqrt{\frac12\E[X]}$, so that 
    \begin{equation}
    \begin{aligned}
         \E[\sqrt X] &\geq \int_{X>\frac12\E[X]}\sqrt X\de \P \\
         &\geq \sqrt{\frac12\E[X]}\int_{X>\frac12\E[X]}\de \P \\
         & = \sqrt{\frac12\E[X]}\P\left( X\geq \frac12\, \E [X]\right) \\
         & \geq \sqrt{\frac12\E[X]}\frac{1}{4K}
    \end{aligned}
    \end{equation}
    where the last inequality comes from~\eqref{eq:lemma_bounded_moments_ineq_pz}.
    Finally, since $\E[X] \geq C$ we get
    \begin{equation}
         \E[\sqrt X] \geq \sqrt{\frac12\E[X]}\frac{1}{4K} \geq \frac{1}{4K}\sqrt{\frac C2}
    \end{equation}
\end{proof}

\section{PROOF OF THEOREM~\ref{thm:minibatch_gradient_norm}}\label{app:proof_minibatch_thm}

To prove Theorem~\ref{thm:minibatch_gradient_norm} we need the following intermediate result.

\begin{lemma}\label{lemma:gradient_norm_bound}
    Let $\pi_\theta(\cdot|s)$ and  $\pi_{\theta_{old}}(\cdot|s)$ be two deep Gaussian policies parameterized by $\theta=(\phi,\sigma)$ and $\theta_{old}=(\phi_{old},\sigma_{old})$, respectively. Assume that the unclipped border $\unclippedbrd$ has probability zero under the state-action measure $\nu^T_{\pi_{\theta_{old}}}$ induced by $\pi_{\theta_{old}}$ and that the following ratio exists and is finite:
    \begin{equation}
        \kappa_\theta(s,a)\coloneqq \frac{\Esa[\|\nabla_\phi\LCLIP(s,a)\|^2] }{\E_{s,a\sim \nu^T_{\pi_{\theta_{old}}}}[\|\nabla_\phi\log\pi_\theta(a|s)\|^2]}.
    \end{equation}
    Then, 
    \begin{equation}\label{eq:gradient_lclip_norm_bound}
    \Esa[\|\nabla_\phi\LCLIP(s,a) \|^2] \geq \kappa_\theta(s,a)\frac{\sigma_{old}^2}{\sigma^4}\Es[\|J_{\mu_\phi}(s)\|_F^2 ].
    \end{equation}
\end{lemma}
\begin{proof}
    Observe that since the set $\unclippedbrd$ has probability zero under $\nu^T_{\pi_{\theta_{old}}}$ the gradients $\nabla_\phi\LCLIP(s,a)$ and $\nabla_\phi\log\pi_\theta(a|s)$ are defined almost everywhere and the integrals used to compute the expectations are well defined.
    Let 
    \begin{equation}\label{eq:xi_defn}
        \xi(s,a)\coloneqq(a-\mu_{\phi}(s))/{\sigma}
    \end{equation}
    and observe that $\xi$ is normally distributed conditionally on $s$ with distribution
    \begin{equation}\label{eq:xi_distribution}
        \xi\sim \cN\left(\frac{\mu_{\phi_{old}}(s)-\mu_\phi(s)}{\sigma}, \frac{\sigma^2_{old}}{\sigma^2}I\,\Bigm\lvert\, s\right).
    \end{equation}
    Let $\mu_\xi\colon\cS\to\cA$ and $\sigma_\xi$ denote the mean and variance of the Gaussian distribution~\eqref{eq:xi_defn}.
    The following direct calculations yield the desired inequality:
    \begin{equation*}
        \begin{aligned}
            \Esa[\|\nabla_\phi\LCLIP(s,a)\|^2] &= \kappa_\theta(s,a)\Esa[\|\nabla_\phi\log\pi_\theta(a|s)\|^2]\\
            &=\kappa_\theta(s,a)\Esxi\left[\left\| \frac{1}{\sigma}J_{\mu_\phi}^T(s)\xi\right\|^2\right] \\
            &=\kappa_\theta(s,a)\frac{1}{\sigma^2}\Es\left[\E_{\xi\sim\cN(\mu_\xi(s), \sigma_\xi)}\left[\left\| J_{\mu_\phi}^T(s)\xi\right\|^2\,\middle|\,s\right]\right]\\
            &=  \kappa_\theta(s,a)\frac{\sigma_{old}^2}{\sigma^4}\Es[\|J_{\mu_\phi}(s)\|_F^2 ]\\  &\quad+\kappa_\theta(s,a)\frac{1}{\sigma^4}\Es[\|J_{\mu_\phi}^\top(s)(\mu_{\phi_{old}}(s)-\mu_\phi(s))\|_2^2]\\
            &\geq \kappa_\theta(s,a)\frac{\sigma_{old}^2}{\sigma^4}\Es[\|J_{\mu_\phi}(s)\|_F^2 ],
        \end{aligned}
    \end{equation*}
    where the second equality comes from~\eqref{eq:log_policy_grad_final} and the third one uses the tower rule. Moreover, the fourth equality comes from the fact that for a fixed matrix $A$ and a standard normally distributed vector $\zeta\sim\cN(0,I)$,
    \[
    \E_\zeta[\|A\zeta\|^2] = \E_\zeta[\|A\|^2_F].
    \]
    By shifting and rescaling $\zeta$ and by linearity of expectation, we get the general formulation for any normal vector used in the equation.
\end{proof}

\begin{proof}[\textbf{Proof of Theorem~\ref{thm:minibatch_gradient_norm}}]
By Lemma~\ref{lemma:gradient_norm_bound}, the per-sample gradient
satisfies
\[
    \Esa[\|\nabla_\phi \LCLIP(s,a)\|^2]
    \geq
    \kappa_\theta
    \frac{\sigma_{old}^2}{\sigma^4}
    \Es\|J_{\mu_\phi}(s)\|_F^2 .
\]
Since the minibatch samples are i.i.d.,
\begin{equation}
    \begin{aligned}
    \EB[\left\|\bar g_\theta(B)\right\|^2]
    &= \E_{s_1,a_1,\ldots,s_N,a_N}\left[\frac{1}{N^2}\left\|\sum_{i=1}^N\nabla_\phi\LCLIP(s_i,a_i)\right\|^2\right]   \\
    &=\frac{1}{N^2}\sum_{i=1}^N\sum_{j=1}^N\E_{s_i,a_i,s_j,a_j}[\nabla_\phi\LCLIP(s_i,a_i)\cdot\nabla_\phi\LCLIP(s_j,a_j)]\\
    &=\frac{1}{N^2}\sum_{i=1}^N\E_{s_i,a_i,s_j,a_j}[\nabla_\phi\LCLIP(s_i,a_i)\cdot\nabla_\phi\LCLIP(s_i,a_i)]\\&
    \quad+ \frac{1}{N^2}\sum_{i=1}^N\sum_{j\neq i}\E_{s_i,a_i,s_j,a_j}[\nabla_\phi\LCLIP(s_i,a_i)\cdot\nabla_\phi\LCLIP(s_j,a_j)]\\
    &= \frac{1}{N^2}\sum_{i=1}^N\E_{s_i,a_i}\left[\|\nabla_\phi \LCLIP(s_i,a_i)\|^2\right] + \frac{N-1}{N^2}\sum_{i=1}^N\left\|\E_{s_i,a_i}[\nabla_\phi \LCLIP(s_i,a_i)]\right\|^2\\
    &=\frac1N \Esa\left[\|\nabla_\phi \LCLIP(s,a)\|^2\right]+\frac{N-1}{N}\left\|\Esa[\nabla_\phi \LCLIP(s,a)]\right\|^2 \\
    &\geq
    \frac1N \Esa\left[\|\nabla_\phi \LCLIP(s,a)\|^2\right].
    \end{aligned}
\end{equation}
Combining this inequality with the result of Lemma~\ref{lemma:gradient_norm_bound} we obtain
\begin{equation}\label{eq:gradient_lower_bound_minibatch}
    \EB[\|\bar g_\theta(B)\|^2]
    \geq
    \frac{\kappa_\theta}{N}
    \frac{\sigma_{old}^2}{\sigma^4}
    \Es[\|J_{\mu_\phi}(s)\|_F^2] .
\end{equation}
Finally, a direct application of Lemma~\ref{lemma:kurtosis} to~\eqref{eq:gradient_lower_bound_minibatch} with $X=\|\bar g_\theta(B)\|^2$ yields the desired result:
\begin{equation}
    \EB[\|\bar g_\theta(B)\|]
    \geq
    \frac{1}{4K_\theta\sqrt{2N}}
    \sqrt{\kappa_\theta}\,
    \frac{\sigma_{old}}{\sigma^2}
    \left(
        \Es[\|J_{\mu_\phi}(s)\|_F^2]
    \right)^{1/2}.
\end{equation}
    
\end{proof}

\section{PROOF OF THEOREM~\ref{thm:lambda_max_bound_squared}}\label{app:proof_lambda_max_thm}
To prove Theorem~\ref{thm:lambda_max_bound_squared} we need the following intermediate result.

\begin{lemma}\label{lemma:moments_gaussian}
    Let $\xi\sim\cN(\mu,\tau^2I)$ be a random vector with Gaussian distribution and let $u$ be a unit vector, i.e., $\|u\|_2 = 1$. Then, the vector $\zeta\coloneqq u ^\top \xi$ is normally distributed with distribution
    \begin{equation}
        \zeta\sim\cN(u^\top\mu, \tau^2)  
    \end{equation}
    and it holds that
    \begin{equation}\label{eq:moments_gaussian_ineq}
        \E_\zeta[(\zeta^2-1)^2]\geq2\tau^4.
    \end{equation}
\end{lemma}
\begin{proof}
    The first part follows from the fact that
    \begin{equation}
        u^T\xi\sim\cN(u^\top\xi, \tau^2 u^\top I u)
    \end{equation}
    and, since $u^\top I u=u^\top u = \|u\|_2^2=1$ we have that 
     \begin{equation}
        \zeta\sim\cN(u^\top\mu,\tau^2).
    \end{equation}
    We now want to show that Equation~\eqref{eq:moments_gaussian_ineq} holds. Let
    \begin{equation}
        \alpha\coloneqq u^\top\mu,
    \end{equation}
    so that $\zeta \sim \cN(\alpha, \tau^2)$. Then, by using the second moments of normally distributed vectors and by linearity of expectation, we obtain:
    \begin{equation}
        \begin{aligned}
        \E_\zeta[(\zeta^2-1)^2] &= \E_\zeta[(\zeta^2-1)^2]\\ &= \E_\zeta[\zeta^4] -2\E_\zeta[\zeta^2] + 1\\
        &=\alpha^4 + 6\alpha^2\tau^2 + 3\tau^4 - 2\alpha^2 - 2\tau^2 + 1\\
        &=(\alpha^2 + \tau^2 -1)^2 + \tau^2(4\mu^2+2\tau^2)\\
        &\geq 2\tau^4
        \end{aligned}
    \end{equation}
\end{proof}
\begin{proof}[\textbf{Proof of Theorem~\ref{thm:lambda_max_bound_squared}}]
Recall the following inequality
\begin{equation}
    \|A+B\|^2
    \geq
    \frac12\|A\|^2-\|B\|^2
\end{equation}
which holds for any norm.
By applying it to $\rho(H_\theta(s,a)^2) =\|H_\theta(s,a)\|^2_\mathrm{op}$, and taking expectations we get
\begin{equation}
    \begin{aligned}
    \Esa[\rho(H_\theta(s,a)^2)]
    &\geq
    \frac{1}{2\sigma^4}\Esa[W_\theta(s,a)^2\|G_\phi(s,\xi)\|_{\mathrm{op}}^2]\\
    &\quad-\frac{1}{\sigma^2} \Esa[W_\theta(s,a)^2\|R(s,\xi)\|_{\mathrm{op}}^2] .
\end{aligned}
\end{equation}
Using the assumptions in~\eqref{eq:chi_C_theta}, we obtain
\begin{equation}\label{eq:hessian_thm_proof_intermediate_value}
    \Esa\left[\rho\left(H_\theta(s,a)^2\right)\right]
    \geq \frac{\chi_\theta}{2\sigma^4}
    \Esxi\left[\|G_\phi(s,\xi)\|_{\mathrm{op}}^2\right]-\frac{C_\theta}{\sigma^2}.
\end{equation}
\noindent Let
\begin{equation}\label{eq:lambda_1_definition}
     \lambda_1\coloneqq\rho(J_{\mu_\phi}(s)J_{\mu_\phi}(s)^\top)
\end{equation}
and let $u_1$ be a unit eigenvector of $J_{\mu_\phi}(s)J_{\mu_\phi}(s)^\top$ associated with $\lambda_1$.\\
Define
\begin{equation}
    v_1 \coloneqq \frac{J_{\mu_\phi}^\top u_1}{\sqrt{\lambda_1}},
\end{equation}
so that
\begin{equation}
    \|v_1\|=1, \qquad J_{\mu_\phi}v_1=\sqrt{\lambda_1}u_1.
\end{equation}
By the Rayleigh quotient,
\begin{equation}
    \|G(s,\xi)\|_{\mathrm{op}} = \max_i |\lambda_i(G(s,\xi))|
    \geq
    \left|
        v_1^\top G_\phi(s,\xi)v_1
    \right| =
    \lambda_1
    \left|
        u_1^\top(\xi\xi^\top-I)u_1
    \right|,
\end{equation}
where $\lambda_i(G_\phi(s,\xi))$ denotes the $i$-th eigenvalue of $G_\phi(s,\xi)$.\\
Recall that $\xi$ is normally distributed, with distribution shown in~\eqref{eq:xi_distribution}.\\
Then, let 
\begin{equation}
    \zeta\coloneqq u_1^\top\xi
\end{equation}
so that the previous equation becomes
\begin{equation}
    \|G(s,\xi)\|_{\mathrm{op}}^2\geq\lambda_1^2(\zeta^2-1)^2.
\end{equation}
By applying Lemma~\ref{lemma:moments_gaussian} we get 
\begin{equation}
    \zeta\sim\mathcal N\left(u_1^\top\frac{\mu_{\phi_{old}}(s)-\mu_\phi(s)}{\sigma},\frac{\sigma^2_{old}}{\sigma^2}\right).
\end{equation}
Taking conditional expectations yields
\begin{equation}\label{eq:operator_eigenvalue_ineq}
    \E_\xi\left[\|G_\phi(s,\xi)\|_{\mathrm{op}}^2\mid s\right]\geq\lambda_1^2\E_\zeta[(\zeta^2-1)^2\mid s].
\end{equation}
This inequality takes the expectation with respect to $\xi$ on the LHS and the expectation with respect to $\zeta$ on the RHS. This is motivated by the fact that we are implicitly using the pushforward measure induced by $\zeta$, using the law of the unconscious statistician.\\
By applying Lemma~\ref{lemma:moments_gaussian} to~\eqref{eq:operator_eigenvalue_ineq} and combining it with~\eqref{eq:lambda_1_definition}, we get
\begin{equation}
    \E_\xi\left[\|G_\phi(s,\xi)\|_{\mathrm{op}}^2\mid s\right]\geq2\frac{\sigma_{old}^4}{\sigma^4}
    \rho\left( J_{\mu_\phi}(s)J_{\mu_\phi}(s)^\top\right)^2.
\end{equation}
Averaging over $s$ and applying the tower rule in the RHS yields
\begin{equation}\Esxi\left[\|G_\phi(s,\xi)\|_{\mathrm{op}}^2\right]\geq2\frac{\sigma_{old}^4}{\sigma^4}
    \Es\left[\rho\left(J_{\mu_\phi}(s)J_{\mu_\phi}(s)^\top\right)^2\right].
\end{equation}
Combining this bound with~\eqref{eq:hessian_thm_proof_intermediate_value} gives the desired inequality
\begin{equation}
    \Esa\left[\rho\left( H_\theta(s,a)^2\right)\right]
    \geq\chi_\theta\frac{\sigma_{old}^4}{\sigma^8}\Es\left[\rho\left(J_{\mu_\phi}(s)J_{\mu_\phi}(s)^\top\right)^2\right]
    - \frac{C_\theta}{\sigma^2}.
\end{equation}
\end{proof}

\section{PROOF OF THEOREM~\ref{thm:hessian_effective_rank_collapse}}\label{app:proof_hess_eff_rank}
\begin{proof}
    Let 
    \begin{equation}
    \begin{aligned}
        M_1&\coloneqq\sum_{i=1}^N W_\theta(s_i,a_i)G_\phi(s_i,\xi_i)\\ M_2&\coloneqq\sum_{i=1}^N W_\theta(s_i,a_i)R_\phi(s_i,\xi_i)
    \end{aligned}
    \end{equation}
    so that 
    \begin{equation}
        \bar H_\theta(B) = \frac{1}{N\sigma^2}(M_1+\sigma M_2)
    \end{equation}
    and 
    \begin{equation}
        \lambda_i(\bar H_\theta(B)) = \frac{1}{N\sigma^2}\lambda_i((M_1+\sigma M_2))
    \end{equation}
    where $\lambda_i(\cdot)$ denotes the $i$-th eigenvalue. 
    Clearly, 
    \begin{equation}
        M_1+\sigma M_2\to M_1,\qquad \text{ as }\sigma\to 0
    \end{equation}
    and by the continuity of the eigenvalues we get
    \begin{equation}
        \lambda_i(M_1+\sigma M_2)\to \lambda_i(M_1), \qquad \text{ as }\sigma\to 0.
    \end{equation}
    Then, for all $j$ such that $\lambda_j(M_1)\neq0$,
    \begin{equation}
         \lambda_j(\bar H_\theta(B)) \sim \frac{1}{N\sigma^2}\lambda_j(M_1), \qquad \text{ as }\sigma\to 0
    \end{equation}
    and for each $k$ such that $\lambda_k(M_1)=0$ we get
    \begin{equation}
         \lambda_k(\bar H_\theta(B)) = o\left(\frac{1}{\sigma^2}\right), \qquad \text{ as }\sigma\to 0
    \end{equation}
    by the continuity of the eigenvalues, so that the eigenvalues of $\bar H_\theta(B)$ scale with $1/\sigma^2$ if and only if they are non-zero eigenvalues of $M_1$.
    Since $M_1$ is not identically equal to 0 and is symmetric, it has at least one non-zero eigenvalue. By definition, 
    \begin{equation*}
        G_\phi(s,\xi)\coloneqq J_{\mu_\phi}(s)^\top(\xi\xi^\top-I)J_{\mu_\phi}(s)
    \end{equation*}
    and since $\xi\in\R^n$, then $\xi\xi^\top\in\R^{n\times n}$.\\
    Recall that for any three matrices $A\in\R^{p\times q}$, $B\in\R^{q\times q}$, $C\in\R^{q\times p}$ it holds 
    \begin{equation}
        \rk{ABC}\leq\min(p,q),
    \end{equation}
    which implies 
    \begin{equation}
        \rk{W_\theta(s,a)G_\phi(s,\xi)}\leq n.
    \end{equation}
    Moreover, by the sub-additivity of the $\mathrm{rank}$ operator, we get
    \begin{equation}\label{eq:M1_upperbound}
    \begin{aligned}
        \rk{M_1} &= \rk{\sum_{i=1}^N W_\theta(s_i,a_i)G_\phi(s_i,\xi_i)}\\ &\leq \sum_{i=1}^N \rk{W_\theta(s_i,a_i)G_\phi(s_i,\xi_i)} \\&\leq \sum_{i=1}^N n= nN
    \end{aligned}
    \end{equation}
    which shows that there are at most $n\times N$ eigenvalues of $\bar H_\theta(B)$ that scale as $1/\sigma^2$.\\
    It remains to show that the stable rank of $\bar H_\theta(B)$ is asymptotically bounded by $\rk{M_1}$ for every $0<\delta<1$.\\
    Let $s_i(\cdot)$ denote the $i$-th singular value, and let
    \begin{equation}
        r\coloneqq \rk{M_1},
    \end{equation}
    then by the continuity of the singular values we have that
    \begin{equation}
        s_i(\bar H_\theta(B)) = \frac{1}{N\sigma^2}s_i(M_1+\sigma M_2) \sim \frac{1}{N\sigma^2}s_i(M_1)\qquad \text{ as $\sigma\to0$,}
    \end{equation}
    for all $i\in\{1,\ldots,r\}$ and 
    \begin{equation}
        s_j(\bar H_\theta(B)) = o\left(\frac{1}{\sigma^2}\right)\qquad \text{ as $\sigma\to0$.}
    \end{equation}
    for all $j\in\{r+1,\ldots,d\}$, where $d$ is the dimensionality of $\phi$. Therefore, 
    \begin{equation}
        \sum_{i=1}^d s_i(\bar H_\theta(B)) = \sum_{i=1}^r s_i(\bar H_\theta(B)) + \sum_{i=r+1}^d s_i(\bar H_\theta(B)) \sim \sum_{i=1}^r\frac{1}{N\sigma^2}s_i(M_1)
    \end{equation}
    as $\sigma\to0$. Therefore,
    \begin{equation}
        \frac{\sum_{i=1}^r s_i(\bar H_\theta(B))}{\sum_{i=1}^d s_i(\bar H_\theta(B))} \sim  \frac{\sum_{i=1}^r s_i(M_1)}{\sum_{i=1}^r s_i(M_1)} \to 1\qquad \text{ as $\sigma\to0$}
    \end{equation}
    which implies that for every $0<\delta<1$, there is a $\sigma_0$ such that
    \begin{equation}
         \frac{\sum_{i=1}^r s_i(\bar H_\theta(B))}{\sum_{i=1}^d s_i(\bar H_\theta(B))} > 1-\delta
    \end{equation}
    for all $\sigma<\sigma_0$. Therefore,
    \begin{equation}
    \limsup_{\sigma\to0}\mathrm{srank}^{(\delta)}(\bar H_\theta(B))\leq r
    \end{equation}
    for every $0<\delta<1$.
\end{proof}

\section{IMPLEMENTATION DETAILS}\label{app:implementation_details}
In this section we provide details over the implementation, the environments and the hyperparameters used, to facilitate the reproducibility of the experiments.
\subsection{Training setup for Slippery Ant}\label{app:sec_training_setup_sant}

This environment was used in~\cite{Dohare2024, dohare2022continual} to study the phenomenon of loss of plasticity in deep RL. It is a modification of the Mujoco Ant task, in which a quadruped needs to learn how to walk forward. The difference with the base task lies in the fact that every 2 million steps, the friction is resampled log-uniformly in the range $[0.02,2.00]$.\\
Below, we report a description of the reward, observation and action space used in this task. Notice that these descriptions match the well-known OpenAI Gym task from which the environment is built.

The \textbf{reward} is given by the following equation

\begin{equation}
r_t =
r_{t,\text{forward}} + r_{t,\text{control}} + r_{t,\text{contact}} + r_{t,{\text{survive}}}
\end{equation}
where $r_{t,\text{forward}}$ rewards the forward motion, $r_{t,\text{control}}$ penalizes abrupt changes in the actions, $ r_{t,\text{contact}}$ penalizes behaviors that result in the ant to experience intense contact forces and  $r_{t,{\text{survive}}}$ rewards the ant for not meeting the episode early termination condition. In fact, an episode will last 1000 timesteps, unless the body height at any given point is not positioned between $0.2$ and $1$. If this condition occurs the episode is immediately terminated.
Specifically, the following values are used:
\begin{equation}
    \begin{aligned}
        r_{t,\text{forward}} &= \frac{x_{t+1}-x_t}{\Delta t}, &\qquad
        r_{t,\text{control}} &= -0.5 \sum_{i=1}^8 a_{i,t}^2, \\
        r_{t,\text{contact}} &= 5 \cdot 10^{-4}\sum_j \left(\mathrm{clip}(F_{\text{ext},j},-1,1)\right)^2, &
        r_{t,\text{survive}} &= 1,
    \end{aligned}
\end{equation}
where $x_t$ and $x_{t+1}$ are the $x$ position at times $t$ and $t+1$, $a_{i,t}$ denotes the action applied to the $i$-th joint and $F_{\text{ext}, j}$ denotes the $j$-th external forces applied to the ant body.\\
The \textbf{observation space} is 111-dimensional, in which the first 27 observations contain the positions and relative angles of the body, while the remaining 84 contain the external forces.\\
Finally, the \textbf{action space} is 8-dimensional, where each entry is the torque applied to the corresponding joint.\\

\begin{table}[t]
\centering
\begin{tabular*}{\textwidth}{@{\extracolsep{\fill}}|ll|}
\hline
\textbf{Hyperparameter} & \textbf{Value} \\
\hline
Actor network shape & $111 \rightarrow 256 \rightarrow 256 \rightarrow 8$ \\
Critic network shape & $111 \rightarrow 256 \rightarrow 256 \rightarrow 1$ \\
Optimizer & Adam\\
Initial policy standard deviation & $1.0$ \\
Activation function & ReLU \\
Discount factor $\gamma$ & $0.99$ \\
GAE parameter $(\lambda)$ & $0.95$ \\
Clipping parameter $\varepsilon$ & $0.2$ \\
Number of mini-batches & $16$ \\
Optimization epochs per iteration & $10$ \\
Entropy coefficient&0\\
Learning rate & $1 \times 10^{-4}$ \\
\hline
\end{tabular*}
\caption{PPO training hyperparameters for Slippery Ant.}
\label{tab:ppo_hyperparameters_slippery_ant}
\end{table}
The studied environments are identical to~\cite{Dohare2024}. We provide here a brief explanation for reproducibility:
\begin{enumerate}
    \item \textbf{Standard PPO: }Adam coefficients are set to $\beta_1=0.9$ and $\beta_2=0.999$, no L2 regularization;
    \item \textbf{Tuned Adam: }Adam coefficients are set to $\beta_1=\beta_2=0.9$, no L2 regularization;
    \item \textbf{Tuned Adam + L2: }Adam coefficients are set to $\beta_1=\beta_2=0.9$,  L2 regularization with parameter $10^{-3}$;
    \item \textbf{Tuned Adam + L2 + CBP: }Adam coefficients are set to $\beta_1=\beta_2=0.9$, L2 regularization with parameter $10^{-4}$ and continual backpropagation, with the following hyperparameters:
    \begin{itemize}
        \item[-] replacement rate: $10^{-4}$;
        \item[-] decay rate: $0.99$;
        \item[-] maturity threshold $1000$.
    \end{itemize}
\end{enumerate}
All the additional training hyperparameters can be found in Table~\ref{tab:ppo_hyperparameters_slippery_ant}.

\subsection{Training setup for \code{Isaac-Reach-Franka-v0}}\label{app:sec_training_setup_franka}
\begin{table}[t]
\centering
\begin{tabular*}{\textwidth}{@{\extracolsep{\fill}}|ll|}
\hline
\textbf{Hyperparameter} & \textbf{Value} \\
\hline
Actor network shape & $32 \rightarrow 64 \rightarrow 64 \rightarrow 7$ \\
Critic network shape & $32 \rightarrow 64 \rightarrow 64 \rightarrow 1$ \\
Optimizer & Adam \\
Initial policy standard deviation & $1.0$ \\
Activation function & ELU \\
Discount factor $\gamma$ & $0.99$ \\
GAE parameter $(\lambda)$ & $0.95$ \\
Clipping parameter $\varepsilon$ & $0.2$ \\
Number of mini-batches& $4$ \\
Optimization epochs per iteration& $8$\\
Entropy coefficient & $0.001$ \\
Initial learning rate & $1\times 10^{-3}$ \\
Learning rate schedule &Adaptive\\
Desired KL divergence for adaptive learning rate & $0.01$ \\
\hline
\end{tabular*}
\caption{PPO training hyperparameters for \code{Isaac-Reach-Franka-v0}.}
\label{tab:ppo_hyperparameters_franka_reach}
\end{table}
In this environment, the task is to solve an inverse kinematics problem, by moving the end-effector to a target pose, while minimizing abrupt changes.
The policies are trained in IsaacLab with 4096 parallel environments. The timesteps are aggregated in \textbf{iterations}. One iteration corresponds to 24 timesteps over all the environments, so that 1 iteration consists of $24\times 4096 = 98304$ timesteps. The \textbf{termination condition} determines the end of the episode. In this environment, an episode ends whenever it lasts longer than 12 seconds, corresponding to 15 iterations.
At the end of each iteration, for all the environments that reached the termination condition, we perform backpropagation for 8 epochs of 4 mini-batches each.

The \textbf{reward function} is given by:
\begin{equation}\label{eq:reward_reach}
\begin{aligned}
r_t =&
-0.2\, \|p_{ee,t}-p^{*}_{ee,t}\|_2 \\
&+ 0.1 \left[
1 - \tanh\left(
\frac{\|p_{ee,t}-p^{*}_{ee,t}\|_2}{0.1}
\right)
\right] \\
&-0.1\, \left\|
\log(\Phi(q_{ee,t}) \Phi^{-1}(q^{*}_{ee,t}))^\vee
\right\|_2 \\
&-k_{l_2}\,\|a_t-a_{t-1}\|_2^2
-k_{vel}\,\|\dot{q}_t\|_2^2 .
\end{aligned}
\end{equation}
where 
\begin{itemize}
    \item $p_{ee,t}$ and $p^{*}_{ee,t}$ denote the current end-effector position and the target pose of the end-effector, respectively;
    \item $q_{ee,t}$ and $q^{*}_{ee,t}$ denote the current end-effector quaternion and the target pose of the end-effector quaternion;
    \item $a_t$ and $a_{t-1}$ denote the current and previous action;
    \item $\dot q_t$ denotes the current joint velocity;
    \item $\Phi(\cdot)\colon \mathbb{H}\to SO(3)$ maps the quaternions to the $3\times 3$ corresponding orthogonal matrix, $\log(\cdot)\colon SO(3)\to \mathfrak{so}(3)$ is the matrix logarithm. Finally, $(\cdot)^\vee\colon \mathfrak{so}(3)\to\mathbb{R}^3$ maps skew symmetric matrices to their three dimensional vector representation.
\end{itemize}
The terms $k_{l_2}$ and $k_{vel}$ are used for curriculum learning. They are initially set to $10^{-4}$. However, after 187 iterations these weights are changed to  $10^{-3}$ and $5\times10^{-3}$, respectively.
The target pose and orientation $p^{*}_{ee,t}, q^{*}_{ee,t}$ are resampled every 4 seconds in the simulation, corresponding to 5 iterations. The sampling happens uniformly in the box defined in Table~\ref{tab:target_pose_sampling}. Notice that the orientation is sampled according to Euler angles with fixed roll and pitch and the resulting orientation is then converted to quaternions.\\
\begin{table}[!t]
    \centering
    \begin{tabular*}{\textwidth}{@{\extracolsep{\fill}}|l l c|}
        \hline
        \textbf{Description} &\textbf{Variable} &  \textbf{Possible Values} \\
        \hline
        Target $x$ position &$p^*_{ee, x}$ [m] & $(0.35,\, 0.65)$ \\
        Target $y$ position &$p^*_{ee, y}$ [m] & $(-0.20,\, 0.20)$ \\
        Target $z$ position &$p^*_{ee, z}$ [m] & $(0.15,\, 0.50)$ \\
        Target roll angle &Roll [rad] & $0$ \\
        Target pitch angle &Pitch [rad] & $\pi$  \\
        Target yaw angle &Yaw [rad] & $(-\pi,\, \pi)$ \\
        \hline
    \end{tabular*}
    \caption{Uniform sampling ranges for the target end-effector pose.}
    \label{tab:target_pose_sampling}
\end{table}
The \textbf{observation space} is described in Table~\ref{tab:obs_reach}. It is 32-dimensional, and it consists in 9 observations about position and 9 observations about velocity. It is then concatenated to the 7-dimensional target pose described by the desired position and quaternion of the end-effector and the 7-dimensional previous action.\\
\begin{table}[!t]
\centering
\resizebox{\textwidth}{!}{
\begin{tabular*}{\textwidth}{@{\extracolsep{\fill}}|c c l|}
\hline
\textbf{Indices} & \textbf{Dimensions} & \textbf{Description} \\
\hline
1-7   & 7 & Arm joint positions \\
8-9   & 2 & Gripper finger positions \\
10-16  & 7 & Arm joint velocities \\
17-18 & 2 & Gripper finger velocities \\
19-25 & 7 & End-effector target pose (Cartesian coordinates and quaternion)\\
25-32 & 7 & Previous action\\
\hline
\end{tabular*}}
\caption{Observation space used in \code{Isaac-Reach-Franka-v0}}
\label{tab:obs_reach}
\end{table}
The \textbf{action space} is given by the desired relative joint positions. That is, the output of the policy is interpreted as desired joints angles, and it is therefore a 7-dimensional continuous action space. Formally, let $q$ denote the current joint position, $q_d$ denote the desired position and $\dot q$ denote the current velocity. Then, the torque $\tau$ applied to each joint is computed as
\begin{equation}\label{eq:stiffness_damping_torque}
    \tau = K_p(q_d-q) - K_d\dot q,
\end{equation}
where the hyperparameters $K_p$ and $K_d$ are commonly named \textit{stiffness} and \textit{damping}, respectively.
Table~\ref{tab:ppo_hyperparameters_franka_reach} contains the hyperparameters used in \code{Isaac-Reach-Franka-v0}.

\section{ADDITIONAL RESULTS}\label{app:additional results}

In this section we report the full experimental evaluation of the theoretical results, including statistical significance tests and plots showing the results of fitting, as well as the changes in gradient magnitude and Hessian sharpness in response to changes in the standard deviation $\sigma$ in \code{Isaac-Reach-Franka-v0}.
Figure~\ref{fig:hessian_and_gradient_fitting} shows how restarting $\sigma$ affects the evolution of the gradient norm and the Hessian largest eigenvalue (or sharpness).
We observe a significant effect, suggesting that $\sigma$ is the main factor affecting these metrics.
\begin{figure}[!t]
    \centering
    \begin{subfigure}{0.49\textwidth}
        \centering
        \includegraphics[width=\linewidth]{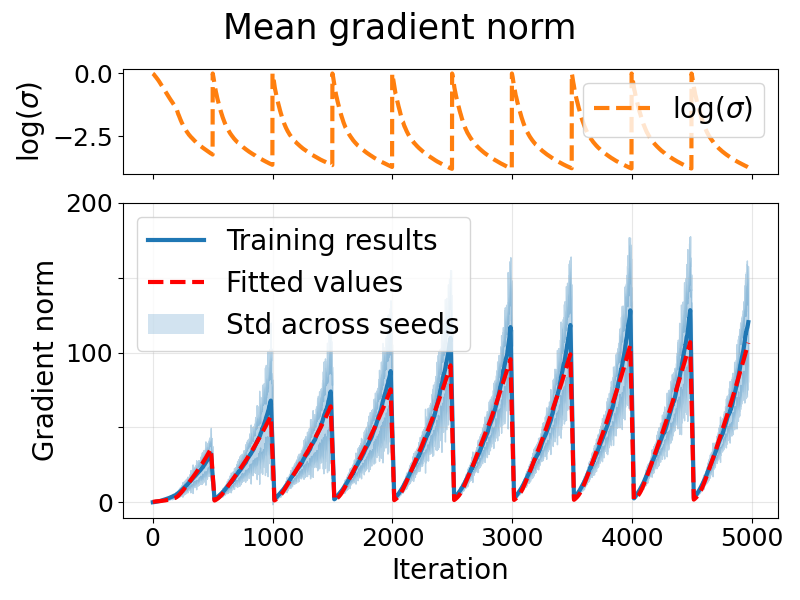}
    \end{subfigure}
    \hfill
    \begin{subfigure}{0.49\textwidth}
        \centering
        \includegraphics[width=\linewidth]{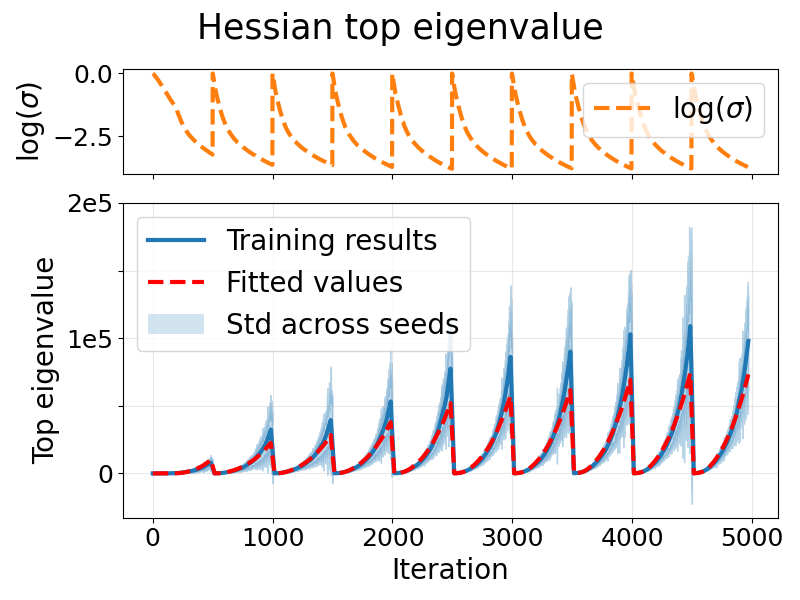}
    \end{subfigure}
    \caption{Evolution of the gradient norm and Hessian sharpness in \code{Isaac-Reach-Franka-v0}, when periodically resetting $\sigma$. Results are obtained over 10 seeds, and the plots are smoothed with a window of 30 iterations.}
    \label{fig:hessian_and_gradient_fitting}
\end{figure}

Table~\ref{tab:grad_fitting_results} contains the fitting results of each 500 iteration window used to fit $\alpha_1$ and $\beta_1$ the regression equation in ~\eqref{eq:log_regression_estimates}. 
We observe that, in the initial training phase, $\alpha_1$ is slightly higher, but then it hovers the value of 1.2. This is consistent with Theorem~\ref{thm:minibatch_gradient_norm}, deriving a bound with exponent 1.
Interestingly, there is a progressive increase in the value of $\beta_1$, which almost doubles suggesting that there are other factors positively contributing to the gradient norm increase.  
Moreover, we always achieve large values of $\mathrm{R}^2$, suggesting that the log-linear fitting is relevant. 

\begin{table}[t]
\centering
\begin{tabular}{|c | c c c c|}
\hline
\textbf{Cycle} & \textbf{Mean }$\mathbf{R^2}$ &\textbf{ Mean log }$\mathbf{R^2}$ & $\mathbf{\alpha_1}$ & $\mathbf{\exp(\beta_1)}$ \\
\hline
1 & 0.8497 & 0.8937 & 1.3182 & 0.5110 \\
2 & 0.8003 & 0.6805 & 1.2057 & 0.7018 \\
3 & 0.7788 & 0.5993 & 1.2045 & 0.7555 \\
4 & 0.7806 & 0.6129 & 1.2200 & 0.7782 \\
5 & 0.7728 & 0.6149 & 1.2447 & 0.8177 \\
6 & 0.7779 & 0.5986 & 1.2428 & 0.8758 \\
7 & 0.7845 & 0.6004 & 1.2408 & 0.9286 \\
8 & 0.7801 & 0.6098 & 1.2423 & 0.9665 \\
9 & 0.7844 & 0.5940 & 1.2380 & 1.0080 \\
10 & 0.7903 & 0.6048 & 1.2323 & 1.0752 \\
\hline
\rule{0pt}{12pt}\textbf{Total} & $0.7899 \pm 0.0377$ & $0.6409\pm 0.0875$ & $1.2389 \pm 0.0594$& $0.8418 \pm 0.2067$ \\[1pt]
\hline
\end{tabular}
\caption{Cycle-wise fitting results for the mean gradient norm. The intervals in the \textit{Total} row indicate one standard deviation.}
\label{tab:grad_fitting_results}
\end{table}

In Table~\ref{tab:hess_fitting_results} we report the fitting results for $\alpha_2$ and $\beta_2$ over each cycle of $\sigma$ reinitialization.
Once again, we observe high values of $\mathrm{R}^2$ and an exponent hovering the value of 2, consistently with Theorem~\ref{thm:lambda_max_bound_squared}.
Moreover, there is a large tenfold increase in the value of $\beta_2$ during training, suggesting that, while $\sigma$ is the most important factor in the sharpness, there are additional components affecting this parameter that are not explained by our model.

\begin{table}[t]
\centering
\begin{tabular}{|c | c c c c|}
\hline
\textbf{Cycle} & \textbf{Mean }$\mathbf{R^2}$ &\textbf{ Mean log }$\mathbf{R^2}$ & $\mathbf{\alpha_2}$ & $\mathbf{\exp(\beta_2)}$ \\
\hline
1 & 0.8250 & 0.9337 & 2.5289 & 2.8127 \\
2 & 0.8345 & 0.9051 & 2.1686 & 7.9413 \\
3 & 0.8285 & 0.8647 & 2.1581 & 9.9758 \\
4 & 0.8183 & 0.8668 & 2.1580 & 11.6019 \\
5 & 0.7937 & 0.8669 & 2.1825 & 13.3129 \\
6 & 0.7861 & 0.8587 & 2.1730 & 15.5046 \\
7 & 0.7960 & 0.8565 & 2.1603 & 17.6697 \\
8 & 0.7974 & 0.8624 & 2.1595 & 19.6723 \\
9 & 0.8019 & 0.8552 & 2.1442 & 22.4115 \\
10 & 0.8128 & 0.8584 & 2.1372 & 25.4125 \\
\hline
\rule{0pt}{12pt}\textbf{Total} & $0.8094 \pm 0.0664$ & $0.8728\pm 0.0245$ & $2.1970 \pm 0.1338$& $14.6315 \pm 7.9061$ \\[1pt]
\hline
\end{tabular}
\caption{Cycle-wise fitting results for the Hessian top eigenvalue. The intervals in the \textit{Total} row indicate one standard deviation.}
\label{tab:hess_fitting_results}
\end{table}

\end{document}